\documentclass[letterpaper]{article}
\usepackage{uai2018}
\usepackage[margin=1in]{geometry}
\usepackage{times}

\usepackage{graphicx}
\usepackage{subcaption}
\usepackage{booktabs} 
\usepackage{mathtools}

\usepackage{color}
\usepackage{natbib}

\usepackage{algorithm}
\usepackage[noend]{algpseudocode}
\usepackage{enumitem}

\usepackage{amsmath, amsthm, amssymb,amsfonts}
\usepackage{url}

\newtheorem{theorem}{Theorem}
\newtheorem{lemma}[theorem]{Lemma}
\newtheorem{proposition}[theorem]{Proposition}
\newtheorem{corollary}{Corollary}
\newtheorem{assumption}{Assumption}
\newcommand{\eproof}{$\null\hfill\blacksquare$}
\renewenvironment{proof}{\par\noindent{\bf Proof:\ }}{\eproof}
\theoremstyle{definition}
\newtheorem{definition}{Definition}[section]

\newcommand{\lb}{\text{LB}}
\newcommand{\ub}{\text{UB}}

\newcommand{\gbar}{\bar{g}}

\newcommand{\erralg}{\bar{\epsilon}}
\newcommand{\errebstop}{\epsilon}
\newcommand{\errstates}{\epsilon_{\nsampledstates}}
\newcommand{\gammaprod}{p_{\gamma}}

\newcommand{\dmu}{{d^\mu}}

\newcommand{\States}{\mathcal{S}}

\newcommand{\E}{\mathbb{E}}

\def\defeq{\buildrel \rm def \over =}

\newcommand{\A}{\mathcal{A}}
\newcommand{\Ss}{\mathcal{S}}
\newcommand{\R}{\mathbb{R}}
\newcommand{\RR}{\mathbb{R}}

\renewcommand{\cite}[1]{\citep{#1}}

\newcommand{\floor}[1]{\left\lfloor #1 \right\rfloor}
\newcommand{\ceil}[1]{\left\lceil #1 \right\rceil}
\newcommand{\sign}[1]{\mathrm{sign}(#1)}

\newcommand{\nsampledstates}{m}
\newcommand{\ntrajectories}{t}
\newcommand{\lenrollout}{l}
\newcommand{\vtrue}{v^*}
\newcommand{\vsampled}{\bar{v}}
\newcommand{\vsampledtrue}{\bar{v}^*}
\newcommand{\vapprox}{\hat{v}}

\newcommand{\stda}{\bar{\sigma}^{(r)}_{\ntrajectories}}
\newcommand{\rmax}{R_{\text{max}}}

\newcommand{\vmax}{V_{\text{max}}}

\newcommand{\relerror}{\tau}

\newcommand{\cliperror}{c}

\newcommand{\clipfs}{\ell_{c}}
\newcommand{\clipf}{\ell}
\newcommand{\clipfhat}{\hat{\ell}}
\newcommand{\losstrue}{\clipf} 
\newcommand{\lossapprox}{\clipfhat}
\newcommand{\CMAVE}{\clipfs}
\newcommand{\CMSVE}{\clipfs} 
\newcommand{\CMAVEtrue}{\losstrue}
\newcommand{\CMSVEtrue}{\losstrue}
\newcommand{\CMAVEapprox}{\lossapprox}
\newcommand{\CMSVEapprox}{\lossapprox}

\newcommand{\epsilonm}{\epsilon_\nsampledstates}
\newcommand{\weighting}{d}

\newcommand{\minc}[1]{\min\left(\cliperror, #1 \right)}

\newcommand{\inv}{{\raisebox{.2ex}{$\scriptscriptstyle-1$}}}

\newcommand{\nexperiments}{K}

\graphicspath{{graphs/}}

\title{High-confidence error estimates for learned value functions} 

\author{Touqir Sajed\\
Department of Computing Science\\
University of Alberta, Canada\\
\texttt{touqir@ualberta.ca}
\And 
Wesley Chung\\
Department of Computing Science\\
University of Alberta, Canada\\
\texttt{wchung@ualberta.ca}
\And Martha White\\
Department of Computing Science\\
University of Alberta, Canada\\
\texttt{whitem@ualberta.ca}
}

\begin{document}
	\maketitle
	\begin{abstract} 
		Estimating the value function for a fixed policy is a fundamental problem in reinforcement learning. Policy evaluation algorithms---to estimate value functions---continue to be developed, to improve convergence rates, improve stability and handle variability, particularly for off-policy learning. To understand the properties of these algorithms, the experimenter needs high-confidence estimates of the accuracy of the learned value functions. For environments with small, finite state-spaces, like chains, the true value function can be easily computed, to compute accuracy. For large, or continuous state-spaces, however, this is no longer feasible. In this paper, we address the largely open problem of how to obtain these high-confidence estimates, for general state-spaces. We provide a high-confidence bound on an empirical estimate of the value error to the true value error. We use this bound to design an offline sampling algorithm, which stores the required quantities to repeatedly compute value error estimates for any learned value function. We provide experiments investigating the number of samples required by this offline algorithm in simple benchmark reinforcement learning domains, and highlight that there are still many open questions to be solved for this important problem. 
	\end{abstract}
	
	\section{INTRODUCTION}
	
	Policy evaluation is a key step in many reinforcement learning systems. Policy evaluation approximates the value of each state---future sum of rewards---given a policy and either a model of the world or a stream of data produced by an agents choices. In classical policy iteration schemes, the agent continually alternates between improving the policy using the current approximation of the value function, and updating the approximate value function for the new policy. Policy search methods like actor-critic estimate the value function of the current policy to perform gradient updates for the policy.
	
	However, there has been relatively little research into methods for accurately evaluating policy evaluation algorithms when the true values are not available.
	In most domains where we are interested in performing policy evaluation, it is difficult or impossible to compute the true value function. We may not have access to the transition probabilities or the reward function in every state, making it impossible to obtain the closed form solution of the true value function $\vtrue$. Even if we have access to a full model of the environment, we may not be able to represent the value function if the number of states is too large or the state is continuous. Aside from small finite MDPs like gridworlds and random MDPs, where closed-form solutions can be computed \citep{geist2014off,white2016investigating}, we often do not have access to $\vtrue$. In nearly all our well-known benchmark domains, such as Mountain Car, Puddle World, Cart pole, and Acrobot, we must turn to some other method to evaluate learning progress.
	
	One option that has been considered is to estimate the objective minimized by the algorithms.
	Several papers \citep{sutton2008dyna,du2017stochastic} have compared the performance of the algorithms in terms of their target objective on a batch of samples, using the approximate linear system for the mean-squared projected Bellman error (MSPBE).
	One estimator, called RUPEE \citep{white2015thesis}, is designed to incrementally approximate the MSPBE by keeping a running average across data produced during learning. Some terms, such as the feature covariance matrix, can be estimated easily; however, one component of the MSPBE includes the current weights, and is biased by this moving average approach. 
	More problematically, some algorithms do not converge to the minimum of the MSPBE, such as residual gradient for the mean-squared Bellman error \citep{baird1995residual} or Emphatic Temporal Difference (ETD) learning \citep{sutton2016anemphatic}, which minimize a variant of the MSPBE with a different weighting. 
	This approach, therefore, is limited to comparing algorithms that minimize the same objective.
	
	The more standard approach has been to use rollouts from states to obtain samples of returns. To obtain these rollout estimates, three parameters need to be chosen: the number of states $\nsampledstates$ from which to rollout, the number of rollouts or trajectories $\ntrajectories$, and the length of each rollout. Given these rollouts, the true values can be estimated from each of the $\nsampledstates$ chosen states, stored offline, and then used for comparison repeatedly during experiments. 
	These evaluation schemes,
	however, have intuitively chosen parameters, 
	without any guarantees that the distance to the true values, the error, is well-estimated. 
	Early work comparing gradient TD algorithms \citep{maei2009convergent} used sampled trajectories---2500 of them---but compared to returns, rather than value estimates. For several empirical studies using benchmark domains, like Mountain Car and Acrobot, there are a variety of choices, including $\ntrajectories = \nsampledstates = 500$ \citep{gehring2016incremental};  $\nsampledstates = 2000$, $\ntrajectories = 300$ and 1000 length rollouts \citep{pan2017accelerated};  and $\nsampledstates = 5000$, $\ntrajectories = 5000$ \citep{le2017learning}.  For a continuous physical system, \citep{dann2014policy} used as little as 10 rollouts from a state.
	Otherwise, other papers have mentioned that extensive rollouts are used\footnote{Note that \citep{boyan1995generalization} used rollouts for a complementary purpose, to train a nonlinear value function, rather than for evaluating policy evaluation algorithms.}, but did not describe how \citep{konidaris2011td,dabney2014natural}.
	In general, as new policy evaluation algorithms are derived, it is essential to find 
	a solution to this open problem: How can we confidently compare value estimates returned by our algorithms?
	
	In this work, we provide an algorithm that ensures, with high-probability, that the estimated distance has small error in approximating the true distance between the true value function $\vtrue$ for an arbitrary estimate $\hat{v}$. 
	We focus in the main body of the paper on the clipped mean-absolute percentage value error (CMAPVE) as a representative example of the general strategy. We provide additional results for a variety of other losses in the appendix, to facilitate use for a broader range of error choices. 
	We conclude by demonstrating the rollout parameters chosen for several case studies, highlighting that previous intuitive choices did not effectively direct sampling. We hope for this algorithm to become a standard approach for generating estimates of the true values to facilitate comparison of policy evaluation algorithms by reinforcement learning researchers.

	%
	%
	%
	%
	%
	%
	%
	%
	

	\section{MEASURES OF LEARNING PERFORMANCE}
	
	This paper investigates the problem of comparing algorithms that estimate the discounted sum of future rewards {\em incrementally} for a fixed policy. 
	In this section, we first introduce the policy evaluation problem and motivate a particular measures of learning performance for policy evaluation algorithms. In the following section, we discuss how to estimate this measure. 
	
	We model the agent's interaction with the world as a Markov decision process (MDP), defined by a (potentially uncountable) set of states $\Ss$, a finite set of actions $\A$, transitions $P: \Ss \times \A \times \Ss \rightarrow [0, \infty)$, rewards $R: \Ss \times \A \times \Ss \rightarrow \RR$ and a scalar discount function $\gamma: \Ss \rightarrow \R$. On each time step $t$, the agent selects an action according to it's {\em behaviour policy} $A_t \sim \mu(\cdot | S_t)$, the environment transitions into a new state $S_{t+1} \sim P(\cdot | S_t, A_t)$ and the agent receives a scalar reward $R_{t+1} \defeq R(S_t, A_t, S_{t+1})$. 
	%
	In policy evaluation, 
	the agent's objective is to estimate the expected return 
	\begin{align}\label{return}
	\vtrue(s) &= \E [G_t | S_t = s, A_t \sim \pi]\\ 
	\text{ for return } G_t &= \sum_{i=0}^\infty \gamma^i R_{t+i + 1} \hspace{0.5cm} \nonumber
	\end{align}
	where $\vtrue(s)$ is called the {\em state-value function} for the {\em target policy} $\pi: \Ss \times \A \rightarrow [0,1]$. 
	From a stream of data, the agent incrementally approximates this value function, $\vapprox$. For experiments, to report learning curves, we need to measure the accuracy of this estimate every step or at least periodically, such as every 10 steps.
	
	For policy evaluation, when the policy remains fixed, the value error remains the gold standard of evaluation. Ignoring how useful the value function is for policy improvement, our only goal is accuracy with respect to $\vtrue$. Assume some weighting $\weighting: \States \rightarrow [0, \infty)$, a probability distribution over states. Given access to $\vtrue$, it is common to estimate the mean squared value error 
	\begin{align}
	\text{MSVE}(\vapprox) &\defeq \E[(\vapprox(S)- \vtrue(S))^2]\\
	&= \int_\Ss \weighting(s)\left(\vapprox(S)- \vtrue(s) \right)^2 ds \nonumber
	\end{align}
	or the mean absolute value error 
	\begin{equation}
	\text{MAVE}(\vapprox)  \defeq \E[\left|\vapprox(s) - \vtrue(s)\right| ]
	.
	\end{equation}
	The integral is replaced with a sum if the set of states is finite. Because we consider how to estimate this error for
	continuous state domains---for which it is more difficult to directly estimate $\vtrue$---we preferentially assume the states are continuous. 
	
	These losses, however, have several issues, beyond estimating them.
	The key issue is that the scale of the returns can be quite different across states. This skews the loss and, as we will see, makes it more difficult to get high-accuracy estimates. 
	Consider a cost-to-goal problem, where the agent receives a reward of -1 per step. From one state the value could be $-1000$ and for another it could be $-1$. For a prediction of $-990$ and $-11$ respectively, the absolute value error for both states would be $-10$. However, the prediction of $-990$ for the first state is quite accurate, whereas a prediction of $-11$ for the second state is highly inaccurate. 
	
	One alternative is to estimate a percentage error, or relative error. The mean absolute percentage value error is
	\begin{equation}
	\text{MAPVE}(\vapprox)  \defeq \E\left[\frac{\left|\vapprox(S) - \vtrue(S)\right|}{ | \vtrue(S)| + \relerror} \right]
	\end{equation}
	for some $\relerror \ge 0$. The term $\relerror$ in the denominator ensures the MAPVE does not become excessively high, if true values of states are zero or near zero. For example, for $\relerror = 1$, the MAPVE is essentially the MAVE for small $\vtrue(s)$, which reflects that small absolute differences are meaningful for these smaller numbers. For large $\vtrue(s)$, the $\relerror$ has little effect, and the MAPVE becomes a true percentage error, reflecting the fact that we are particularly interested in relative errors for larger $\vtrue(s)$. 
	
	Additionally, the MAPVE can be quite large if $\vapprox$ is highly inaccurate. When estimating these performance measures, however, it is uninteresting to focus on obtaining high-accuracy estimate of very large MAPVE. Rather, it is sufficient to report that $\vapprox$ is highly inaccurate, and focus the estimation of the loss on more accurate $\vapprox$. Towards this goal, we introduce the clipped MAPVE
	\begin{align*}
	\text{CMAPVE}(\vapprox)  \defeq \E\left[\minc{\frac{\left|\vapprox(S) - \vtrue(S)\right|}{ | \vtrue(S)| + \relerror}}\right]
	\end{align*}
	for some $\cliperror > 0$. This $\cliperror$ provides a maximum percentage error. For example, setting $\cliperror = 2$ caps error estimates for approximate values that are worse than $200\%$ inaccurate. Such a level of inaccuracy is already high, and when comparing policy evaluation algorithms, we are much more interested in their percentage error---particularly compared to each other---once we are within a reasonable range around the true values. Note that $\cliperror$ can be chosen to be the maximum value of the loss, and so the following results remain quite general. 
	
	Though many losses could be considered, 
	we put forward the CMAPVE as a proposed standard for policy evaluation. 
	The parameters $\relerror$ and $\cliperror$ can both be appropriately chosen by the experimentalist, for a given MDP. These parameters give sufficient
	flexibility in highlighting differences between algorithms, while still enabling high-confidence estimates of these errors, which we discuss next. For this reason, we use CMAPVE as the loss in the main body of the text. However, for completeness, we also show how to modify the analysis and algorithms for other losses in the appendix.

	\section{HIGH-CONFIDENCE ESTIMATES OF VALUE ERROR}

	Our goal now is to approximate the value error, CMAPVE, with high-confidence, for any value function $\vapprox$.
	Instead of approximating the error directly for each $\vapprox$, the typical approach is to estimate $\vtrue$ as accurately as possible, for a large set of states $s_1, \ldots, s_\nsampledstates \sim \weighting$. Given these high-accuracy estimates $\vsampled$, the true expected error can be approximated from this subset of states for any $\vapprox$.
	\begin{equation*}
	\text{CMAPVE}(\vapprox)  \approx \frac{1}{\nsampledstates}\sum_{i=1}^\nsampledstates \minc{\frac{\left|\vapprox(s_i) - \vsampled(s_i)\right| }{|\vsampled(s_i)| + \relerror}}
	\end{equation*}
	Since the CMAPVE needs to be computed frequently, for many steps during learning potentially across many algorithms, it is important
	for this estimate of CMAPVE to be efficiently computable. An important requirement, then, is for the number of states $\nsampledstates$ to be as small as possible,
	so that all the $\vsampled(s_i)$ can be stored and the summed difference is quick to compute. 
	
	One possible approach is to estimate the true value function $\vtrue$ using a powerful function approximator, offline. 
	A large batch of data could be gathered, and a learning method used to train $\vsampled$. This large function approximator would not even need to be stored: only $\vsampled(s_i)$ would need to be saved once this offline procedure was complete. This approach, however, will be biased by the form of the function approximator, which can favor certain policy evaluation algorithms during evaluation. Further, it is difficult to quantify this bias, particularly in a general way agnostic to the type of function approximator an experimentalist might use for their learning setting. 
	
	An alternative strategy is to use many sampled rollouts from this subset of states. This strategy is general---requiring only access to samples from the MDP. A much larger number of interactions can be used with the MDP, to compute $\vsampled$, because this is computed once, offline, to facilitate many further empirical comparisons between algorithms after. For example, one may want to examine the early learning performance of two different policy evaluation algorithms---which may themselves receive only a small number of samples. The cached $\vsampled$ then enables computing this early learning performance. However, even offline, there are limits to how many samples can be computed feasibly, particularly for computationally costly simulators \citep{dann2014policy}.
	
	Therefore, our goal is the following: how can we \emph{efficiently} compute \emph{high-confidence} estimates of CMAPVE, using a \emph{minimal number of offline rollouts}. The choice of a clipped loss actually enables the number of states $\nsampledstates$ to remain small (shown in Lemma \ref{lem_one}), enabling efficient computation of CMAPVE. In the next section, we address the second point: how to obtain high-confidence estimates, given access to $\vsampled$ that approximates $\vtrue$. In the following section, we discuss how to obtain these $\vsampled$.

	\subsection{OVERVIEW}
	
	We first provide an overview of the approach, to make it easier to follow the argument. 
	We additionally include a notation table (Table \ref{table_notation}), particularly to help discern the various value functions.
	
	\setlength\tabcolsep{1.5pt} 
	\begin{table}[h]\caption{Table of Notation}
		
		\begin{tabular}{l c p{7cm} }
			\toprule
			$\vtrue$ & &  true values for policy $\pi$\\
			$\vsampledtrue$ & &  true values for policy $\pi$,\\
			&  & when using truncated rollouts to length $\lenrollout$\\
			$\vsampled$ & &  estimated values for policy $\pi$ using $\ntrajectories$ rollouts,\\
			& & when using truncated rollouts to length $\lenrollout$\\
			$\vapprox$ & &  estimated values for policy $\pi$,  being evaluated\\
			$d$ & &  distribution over the states $\States$, $d: \States \rightarrow [0, \infty)$\\
			$\nsampledstates$ & &  number of states $\{s_1, \ldots, s_\nsampledstates\}$, $s_i \sim d$\\
			$\clipfs$ & &  clipped error, $\clipfs(v_1, v_2) = \minc{\frac{\left|v_1 - v_2\right| }{|v_2| + \relerror}}$\\
			$\clipf$ & &  true error, $\clipf(\vapprox, \vtrue) = \E[ \clipfs(\vapprox(S), \vtrue(S)) ]$ under $d$\\
			$\clipfhat$ & &  approximate error, \\
			& & $\clipfhat(\vapprox, \vtrue) =\frac{1}{\nsampledstates} \sum_{i=1}^\nsampledstates \clipfs(\vapprox(s_i), \vtrue(s_i))$\\
			$\rmax$ & & an upper bound on the maximum absolute value reward, $\rmax \ge \sup | R(s, a, s') |$\\
			$\vmax$ & &  maximum absolute value for the policy $\pi$ for any state, e.g., $\vmax = \frac{\text{max reward} - \text{min reward}}{1-\gamma}$\\
			$K$ & &  the number of times the error estimate is queried\\
			\bottomrule
		\end{tabular}
		
		\label{table_notation}
	\end{table}
	
	First, we consider several value function approximations, for use within the bound, summarized in Table \ref{table_notation}. 
	The goal is to determine the accuracy of the estimates of the learned $\vapprox$ with respect to the true values $\vtrue$. We estimate true values $\vtrue(s_i)$ for $s_i$ using repeated rollouts from $s_i$; this results in two forms of error. The first is due to truncated rollouts, which for the continuing case would otherwise be infinitely long. The second source of error is due to using an empirical estimate of the true values, by averaging sampled returns. We denote $\vsampledtrue$ as the true values, for truncated returns, and $\vsampled$ as the sample estimate of $\vsampledtrue$ from $\nsampledstates$ truncated rollouts. 
	
	Second, we consider the approximation in computing the loss: the difference between $\vapprox$ and $\vtrue$. We consider the true loss $\losstrue(\vapprox, \vtrue)$ and the approximate loss $\lossapprox(\vapprox, \vtrue)$, in Table \ref{table_notation}. The argument in Theorem \ref{thm_main} revolves around upper bounding the difference between these two losses, in terms of three terms. These terms are bounded in Lemmas 2, 3 and 4. Lemma 2 bounds the error due to sampling only a subset of $\nsampledstates$. Lemma 3 bounds the error from approximating $\vtrue$ with truncated rollouts. Lemma 4 bounds the error from dividing by $|\vsampled(s_i)|+\tau$ instead of $|\vtrue(s_i)|+\tau$. 

	Finally, to obtain this general bound, we first assume that we can obtain highly-accurate estimates $\vsampled$ of $\vsampledtrue$. We state these two assumptions in Assumptions \ref{ass_one} and \ref{ass_two}. These estimates could be obtained with a variety of sampling strategies, and so separate it from the main proof. We later develop one algorithm to obtain such estimates, in Section \ref{sec_rollouts}.

	\subsection{MAIN RESULT}
	
	We will compute rollout values from a set of sampled states $\{s_i \}_{i=1}^\nsampledstates$. Each rollout consists of a trajectory simulated, or rolled out, some number of steps. The length of this trajectory can itself be random, depending on if an episode terminates or if the trajectory is estimated to be a sufficiently accurate sample of the full, non-truncated return. We first assume that we have access to such trajectories and rollout estimates and in later sections show how to obtain such trajectories and estimates. 
	
	\begin{assumption}\label{ass_one}
		For any $\epsilon > 0$ and sampled state $s_i$, 
		the trajectory lengths $\lenrollout_i$ are specified such that, 
		\begin{equation}
		| \vsampledtrue(s_i) -  \vtrue(s_i) | \le \epsilon \left(| \vtrue(s_i) | + \relerror \right) \label{eq_assone}
		\end{equation}
	\end{assumption}
	
	\begin{assumption}\label{ass_two}
		Starting from $s_i$, assume you have trajectories of rewards $\{r_{ijk}\}$ for trajectory index $j \in \{1, \ldots, \ntrajectories\}$ and rollout index $k \in \{0, \ldots, \lenrollout_{ij}-1\}$ for a trajectory length $\lenrollout_{ij}$ that depends on the trajectory.
		The approximated rollout values  
		\begin{align}
		\vsampled(s_i) &\defeq \frac{1}{\ntrajectories} \sum_{j = 1}^\ntrajectories \sum_{k=0}^{\lenrollout_{ij}-1} \gamma^k r_{ijk}
		\end{align}
		are an ($\errebstop,\delta,\relerror$) -approximation to the true expected values, where $\lenrollout_{ij}$ is an instance of the random variable $\lenrollout_i$
		\begin{align}
		\vsampledtrue(s) &\defeq \E\left[\sum_{k=0}^{\lenrollout_i-1} \gamma^k R_{k} \right]
		\end{align}
		i.e,  for $0 < \epsilon$, with probability at least $1-\delta/2$, the following holds for all states
		\begin{equation}
		| \vsampledtrue(s_i) -  \vsampled(s_i) | \le \epsilon \left(| \vsampledtrue(s_i) | + \relerror \right)  \label{eq_asstwo}
		\end{equation}
	\end{assumption}

	\begin{theorem}
		\label{thm_main}
		Let $\{ s_1, \ldots, s_\nsampledstates \}$ be states sampled according to $\weighting$. 
		Assume $\vsampled(s_i)$ satisfies Assumption \ref{ass_one} and \ref{ass_two} for all $i \in \{1, \ldots, \nsampledstates\}$. Suppose the empirical loss mean estimates are computed $\nexperiments$ number of times with different learned value functions $\vapprox$ each time.
		Then the approximation error 
		\begin{equation}
		\clipfhat(\vapprox,\vsampled)  \defeq \frac{1}{\nsampledstates} \sum_{i = 1}^\nsampledstates \clipfs(\vapprox(s_i),  \vsampled(s_i)) 
		\end{equation}
		for all the $\vapprox$ satisfies, with probability at least $1-\delta$,
		\begin{equation}
		\left| \clipf(\vapprox,\vtrue)  -  \clipfhat(\vapprox,\vsampled) \right| \le \eqref{eq_one} + \eqref{eq_two} + \eqref{eq_four}
		\end{equation}
	\end{theorem}
	\begin{proof}
		We need to bound the errors introduced from having a reduced number of states, a finite set of trajectories
		to approximate the expected returns for each of those states and truncated rollouts to get estimates of returns. 
		To do so, we first consider the difference under the approximate clipped loss, to the true value function.
		\begin{align*}
		 \!\left| \clipf(\vapprox,\vtrue\!)   \!-\! \clipfhat(\vapprox, \vsampled)  \right|
		&\!\le\! 
		\left| \clipf(\vapprox,\vtrue\!)   \!-\! \clipfhat(\vapprox, \vtrue\!)  \right| \!\!+\!\! \left| \clipfhat(\vapprox,\vtrue\!)  \! -\! \clipfhat(\vapprox, \vsampled)  \right|
		\end{align*}
		The first component is bounded in Lemma \ref{lem_one}.
		For the second component, notice that
		\begin{align*}
		\!\left| \clipfhat(\vapprox,\vtrue\!) \! -\! \clipfhat(\vapprox, \vsampled)  \right|
		&\!\le\!  \tfrac{1}{\nsampledstates}\!\! \sum_{i = 1}^\nsampledstates \!\left|  \clipfs( \vapprox(\!s_i\!)\!, \vtrue\!(\!s_i\!))  \!-\! \clipfs( \vapprox(\!s_i\!), \vsampled(\!s_i\!) )  \!\right|
		\end{align*}
		However, these two differences are difficult to compare, because they have different denomiators: the first has
		$|\vtrue(s_i)| + \relerror$, whereas the second has $|\vsampled(s_i)| + \relerror$.
		We therefore further separate each component in the sum
		\begin{align*}
		&\Big| \clipfs( \vapprox(s_i), \vtrue(s_i) )  - \clipfs( \vapprox(s_i), \vsampled(s_i) ) \Big|\\
		&\le 
		\Big| \clipfs( \vapprox(s_i), \vtrue(s_i) )  - \min\left(\cliperror, \frac{|\vapprox(s_i) - \vsampled(s_i)|}{|\vtrue(s_i)| + \relerror} \right) \Big|\\
		&+
		\Big|\min\left(\cliperror, \frac{|\vapprox(s_i) - \vsampled(s_i)|}{|\vtrue(s_i)| + \relerror} \right) - \clipfs( \vapprox(s_i), \vsampled(s_i) ) \Big|
		\end{align*}
		The first difference has the same denominator, so
		\begin{align*}
		&\Big| \clipfs( \vapprox(s_i), \vtrue(s_i) )  - \min\left(\cliperror, \frac{|\vapprox(s_i) - \vsampled(s_i)|}{|\vtrue(s_i)| + \relerror} \right) \!\Big|\\
		&=\!
		\Big| \! \min\left(\!\cliperror, \frac{|\vapprox(s_i) \!-\! \vtrue(s_i)|}{|\vtrue(s_i)| + \relerror} \!\right) \! -\! \min\left(\!\cliperror, \frac{|\vapprox(s_i) \!-\! \vsampled(s_i)|}{|\vtrue(s_i)| + \relerror} \!\right) \! \Big|\\
		&=
		\frac{1}{|\vtrue(s_i)| + \relerror} \Big| \min\left(\cliperror (|\vtrue(s_i)| + \relerror), |\vapprox(s_i) - \vtrue(s_i)| \right)  \\
		&- \min\left(\cliperror (|\vtrue(s_i)| + \relerror), |\vapprox(s_i) - \vsampled(s_i)| \right) \Big|\\
		&\le
		\frac{1}{|\vtrue(s_i)| + \relerror}\min\left(\cliperror (|\vtrue(s_i)| + \relerror), |\vtrue(s_i) - \vsampled(s_i)| \right)\\
		&=
		\clipfs( \vsampled(s_i), \vtrue(s_i) ) 
		\end{align*}
		The last step follows from the triangle inequality $| \ |x| - |y| \ |_\cliperror \le | x - y|_\cliperror$
		on the clipped loss (see Lemma \ref{lem_triangle}, in Appendix \ref{app_proofs}, for an explicit proof that the triangle inequality holds for the clipped loss).  
		
		Therefore, putting it all together, we have
		\begin{align*}
		& \left| \clipf(\vapprox,\vtrue)   - \clipfhat(\vapprox, \vsampled)  \right|\\
		&\le  \left| \clipf(\vapprox,\vtrue)   - \clipfhat(\vapprox, \vtrue)  \right| \\
		&+  \frac{1}{\nsampledstates} \sum_{i = 1}^\nsampledstates \clipfs( \vsampled(s_i), \vtrue(s_i) ) \\
		&+ \frac{1}{\nsampledstates} \sum_{i = 1}^\nsampledstates \Big|\min\left(\cliperror, \frac{|\vapprox(s_i) - \vsampled(s_i)|}{|\vtrue(s_i)| + \relerror} \right) - \clipfs( \vapprox(s_i), \vsampled(s_i) ) \Big|
		\end{align*}
		where the first, second and third components are bounded in Lemmas \ref{lem_one}, \ref{lem_two} and \ref{lem_three} respectively. Finally, due to the application of Hoeffding's bound (Lemma \ref{lem_one}) with error probability of atmost $\delta/2$ and assumption 2 which may not hold with probability atmost $\delta/2$ and the union bound, we conclude that the final bound holds with probability at least $1- \delta$.
	\end{proof}
	
	%
	
	\begin{lemma}[Dependence on $\nsampledstates$]\label{lem_one}
		
		Suppose the empirical loss mean estimates are computed $\nexperiments$ number of times. Then with probability at least $1 - \delta/2$:
		\begin{equation}
		\left| \clipf(\vapprox,\vtrue)  - \clipfhat( \vapprox, \vtrue ) \right| \le \sqrt{\frac{\log(4\nexperiments/\delta) \cliperror^2}{2\nsampledstates}} \label{eq_one}
		\end{equation} 
	\end{lemma}
	\begin{proof}
		Since $\clipfhat( \vapprox, \vtrue ) =  \frac{1}{\nsampledstates} \sum_{i = 1}^\nsampledstates \clipfs( \vapprox(s_i), \vtrue(s_i) ) $ is an unbiased estimate of $\clipf(\vapprox,\vtrue)$, we can use Hoeffding's bound for variables bounded between $[0, \cliperror]$. For any of the $\nexperiments$ times, the concentration probability is as follows:
		\begin{align*}
		\text{Pr} \bigg( \! \left| \clipf(\vapprox,\vtrue) \!-\!  \clipfhat( \vapprox, \vtrue ) \right| \! \geq \! t  \! \bigg) 
		&\leq 2 \text{ exp} \!\left( \cfrac{-2t^2 \nsampledstates^2 }{\sum_{i=1}^{m} \cliperror^2} \right) \\
		&= 2 \text{ exp} \!\left( \cfrac{-2\nsampledstates t^2}{\cliperror^2} \right) = \frac{\delta}{2 \nexperiments} 
		\end{align*}
		Thus, due to union bound over all the $\nexperiments$ times, for all those empirical loss mean estimates, the following holds  
		\begin{align*}
		Pr \bigg( \left| \clipf(\vapprox,\vtrue) - \clipfhat( \vapprox, \vtrue ) \right|  \leq t \bigg) 
		\geq 1 - \delta/2.
		\end{align*}
		Rearranging the above, to express $t$ in terms of $\delta$, 
		\begin{align*}
		2 \text{ exp} \left( \cfrac{-2mt^2}{c^2} \right) = \frac{\delta}{2 \nexperiments} 
		\ \ \ \implies  \ \ \ 
		t &= \sqrt{\frac{\log(4\nexperiments/\delta) \cliperror^2}{2\nsampledstates}} .
		\end{align*}
		Therefore, with probability at least $1-\delta$, \\
		$\left| \clipf(\vapprox,\vtrue)  - \clipfhat( \vapprox, \vtrue ) \right| \le \sqrt{\cfrac{\log(4\nexperiments/\delta) \cliperror^2}{2\nsampledstates}}$.
		
	\end{proof} 
	
	\begin{lemma}[Dependence on Truncated Rollout Errors]\label{lem_two}
		Under Assumption \ref{ass_two} and \ref{ass_one}, 
		\begin{equation}
		\frac{1}{\nsampledstates} \sum_{i = 1}^\nsampledstates \clipfs( \vsampled(s_i), \vtrue(s_i) ) 
		\le 2\epsilon
		\label{eq_two}
		\end{equation}
	\end{lemma}
	\begin{proof}
		We can split up this error into sampling error for a finite length rollout
		and for a finite number of trajectories. We can consider the unclipped error,
		which is an upper bound on the clipped error.
		\begin{align*}
		\frac{| \vtrue(s_i) -  \vsampled(s_i)|}{| \vtrue(s_i)| + \relerror} &\le \frac{|\vtrue(s_i) -  \vsampledtrue(s_i) |}{| \vtrue(s_i)| + \relerror} + \frac{| \vsampledtrue(s_i) - \vsampled(s_i) |}{| \vtrue(s_i)| + \relerror} 
		\end{align*}
		These two terms are both bounded by $\epsilon$, by assumption. 
	\end{proof}
	\begin{lemma}\label{lem_three}
		Under Assumption \ref{ass_two},
		\begin{align}
		&\Big|\min\left(\cliperror, \frac{|\vapprox(s_i) - \vsampled(s_i)|}{|\vtrue(s_i)| + \relerror} \right) - \min\left(\cliperror, \frac{|\vapprox(s_i) - \vsampled(s_i)|}{|\vsampled(s_i)| + \relerror} \right)\Big| \nonumber\\
		&\le c(1-(1+\epsilon)^{-2})
		\label{eq_four}
		\end{align}
	\end{lemma}
	\begin{proof}
		We need to bound the difference due to the difference in normalizer. To do so, we simply need to find a constant $\beta > 0$ such that
		$\min\left(\cliperror, \frac{|\vapprox(s_i) - \vsampled(s_i)|}{|\vtrue(s_i)| + \relerror} \right) \ge \beta$
		
		The key is to lower bound $|\vtrue(s_i)| + \relerror$, which results in an upper bound on the first term and consequently an upper bound on the difference between the two terms.
		\begin{align*}
		\frac{|\vsampled(s_i)| + \relerror}{|\vtrue(s_i)| + \relerror} 
		&\le \frac{|\vsampled(s_i) - \vsampledtrue(s_i)| + |\vsampledtrue(s_i)| + \relerror}{|\vtrue(s_i)| + \relerror} \\
		&= \frac{|\vsampled(s_i) - \vsampledtrue(s_i)|}{|\vtrue(s_i)| + \relerror} + \frac{|\vsampledtrue(s_i)| + \relerror}{|\vtrue(s_i)| + \relerror} \\
		&\le \frac{\epsilon (| \vsampledtrue(s_i) | + \relerror)}{{| \vtrue(s_i)| + \relerror}} + \frac{|\vsampledtrue(s_i)| + \relerror}{|\vtrue(s_i)| + \relerror} \\
		&= \frac{(1+\epsilon) (| \vsampledtrue(s_i) | + \relerror)}{{| \vtrue(s_i)| + \relerror}} 
		\end{align*}
		where the second inequality is due to Assumption \ref{ass_two}.
		Now further
		\begin{align*}
		\frac{ (| \vsampledtrue(s_i) | + \relerror)}{| \vtrue(s_i)| + \relerror} 
		&\le  \frac{ | \vsampledtrue(s_i) - \vtrue(s_i) |}{| \vtrue(s_i)| + \relerror}  + \frac{ |\vtrue(s_i) | + \relerror}{| \vtrue(s_i)| + \relerror} \\
		&\le  \frac{ \epsilon(|\vtrue(s_i)| + \relerror)}{| \vtrue(s_i)| + \relerror}  + 1\\
		&=  \epsilon + 1
		\end{align*}
		giving
		\begin{align*}
		\frac{|\vsampled(s_i)| + \relerror}{|\vtrue(s_i)| + \relerror} &\le (1+\epsilon)^2 \implies 
		|\vtrue(s_i)| + \relerror \ge 
		\frac{|\vsampled(s_i)| + \relerror}{(1+\epsilon)^2}.
		\end{align*}

		So, for $a = (1+\epsilon)^2$ and $b = \frac{|\vapprox(s_i) - \vsampled(s_i)|}{|\vsampled(s_i)| + \relerror}$, 
		the term $|\minc{ab} - \minc{b}|$ upper bounds the difference. 
		Because $a > 1$ and $b > 0$, this term $|\minc{ab} - \minc{b}|$ is maximized when $a b = c$, and $b = c/a$. In the worst case, therefore, $|\min(c,ab) - \min(c,b)| \le c(1-a^\inv)$ which finishes the proof.  
	\end{proof}
	
	
	\subsection{SATISFYING THE ASSUMPTIONS}
	
	The bounds above relied heavily on accurate sample estimates of $\vtrue(s_i)$. 
	To obtain Assumption \ref{ass_one}, we need to rollout trajectories sufficiently far to ensure that truncated sampled returns
	do not incur too much bias.
	For problems with discounting, for $\gamma < 1$, the returns can be truncated once $\gamma^\lenrollout$ becomes sufficiently small,
	as the remaining terms in the sum for the return have negligible weight.
	For episodic problems with no discounting, it is likely that trajectories need to be
	simulated until termination, since rewards beyond the truncation horizon would not be discounted and so could have considerable weight.
	
	We show how to satisfy Assumption \ref{ass_one}, for the discounted setting. Note that for the trivial setting of $\rmax = 0$, it is sufficient to use
	$\lenrollout = 1$, so we assume $\rmax > 0$.  
	\begin{lemma}\label{lem_discounted}
		For $\gamma < 1$ and $\rmax > \epsilon \tau (1-\gamma)$, if 
		\begin{equation}
		\lenrollout = \left\lceil \frac{\log(\epsilon \tau (1-\gamma)) - \log(\rmax)}{\log \gamma} \right\rceil \label{eq_rollout}
		\end{equation}
		then $\vsampledtrue(s)$ satisfies Assumption \ref{ass_one}: 
		\begin{equation}
		| \vsampledtrue(s) -  \vtrue(s) | \le \epsilon \left(| \vtrue(s) | + \relerror \right)
		\end{equation}
	\end{lemma}
	\begin{proof}
		The first component can be bounded as
		\begin{align*}
		| \vsampledtrue(s) - \vtrue(s) | &\le \left| E\left[\sum_{k=0}^\infty \gamma^{k} R_{k}\right] - E\left[ \sum_{k=0}^{\lenrollout -1} \gamma^{k} R_k\right] \right|\\
		&\le \rmax\left(\frac{1}{1-\gamma} -  \frac{1-\gamma^{\lenrollout}}{1-\gamma} \right)\\
		&= \rmax \frac{\gamma^{\lenrollout}}{1-\gamma} 
		\end{align*}
		giving
		\begin{align*}
		\frac{| \vsampledtrue(s) - \vtrue(s_i) |}{| \vtrue(s)| + \relerror}
		&\le  \rmax \frac{\gamma^{\lenrollout}}{\tau(1-\gamma)}
		.
		\end{align*}
		Setting $\lenrollout$ as in \eqref{eq_rollout} ensures $\rmax \frac{\gamma^{\lenrollout}}{\tau(1-\gamma)} \le \epsilon$,
		completing the proof. 
	\end{proof}
	
	For Assumption 2, we need a stopping rule for sampling truncated returns that ensures $\vsampled$ is within $\epsilon$
	of the true expected value of the truncated returns, $\vsampledtrue$. The idea is to continue sampling truncated returns, until the confidence interval around the mean estimate shrinks sufficiently to ensure, with high probability, that the values estimates are within $\epsilon$ of the true values. Such stopping rules have been designed for non-negative random variables \citep{domingos2001ageneral,dagum2006anoptimal},
	and extended to more general random variables \citep{mnih2008empirical}. 
	We defer the development of such an algorithm for this setting until the next section.
	
	
	\section{THE ROLLOUT ALGORITHM}\label{sec_rollouts}
	
	We can now design a high-confidence algorithm for estimating the accuracy of a value function. 
	Practically, the most important number to reduce is $\nsampledstates$,
	because these values will be stored and used for comparisons on each step.
	The choice of a clipped loss, however, makes it more manageable to control $\nsampledstates$. 
	In this section, we focus more on how much the variability in trajectories, and trajectory length, impact the number of required samples.
	
	The general algorithm framework is given in Algorithm 1.
	The algorithm is straightforward once given an algorithm to sample rollouts from a given state. 
	The rollout algorithm is where development can be directed, to reduce the required number of samples. 
	This rollout algorithm needs to be designed to satisfy Assumptions \ref{ass_one} and \ref{ass_two}. 
	We have already shown how to select trajectory lengths to satisfy Assumption \ref{ass_one}. 
	Below, we describe how to select $\nsampledstates$ and how to satisfy Assumption \ref{ass_two}.

	\begin{algorithm}[htp!]
		\caption{Offline computation of $\vsampled$, to get high-confidence estimates of value error}
		\label{alg_error}
		\begin{algorithmic}[1]
			\State $\triangleright$ Input $\errebstop,\delta, \tau$
			\State $\triangleright$ Compute the values $\vsampled$ once offline and store for repeated use
			\State Set $\epsilonm = \epsilon/2$ and $\erralg = \epsilon/(2 (1+\cliperror))$
			\State $ \nsampledstates \gets \frac{ \log(4 \nexperiments/\delta) \cliperror^2}{2\errstates^2}$
			\For {$1, \ldots, \nsampledstates$}
			\State Sample $s_i \sim d$
			\State $\vsampled(s_i) \gets$ Algorithm \ref{alg_ebgstop} with $\erralg, \frac{\delta}{2\nsampledstates}, \relerror$   
			\EndFor
		\end{algorithmic}
	\end{algorithm}
	
	\newcommand{\lbtrue}{\widehat{\lb}}
	\newcommand{\ubtrue}{\widehat{\ub}}
	
	\begin{algorithm}[t!]
		\caption{High-confidence Monte carlo estimate of the expected return for a state}
		\label{alg_ebgstop}
		\begin{algorithmic}[1]
			\State $\triangleright$ Input $\epsilon,\delta, \relerror$, state $s$
			\State $\triangleright$ Output an $\epsilon,\delta,\relerror$-accurate approx. $\vsampled(s_i)$ of $\vtrue(s_i)$
			\State $\lb \gets 0$,  $\ub \gets \infty$
			\State $\lbtrue \gets -\infty$,  $\ubtrue \gets \infty$
			\State $\gbar \gets 0$, $M \gets 0$ 
			\State $j \gets 1$, $h \gets 0$, $\beta \gets 1.1$, $p \gets 1.1$, $\alpha \gets 1$, $x \gets 1$
			\While {$(1+\epsilon)\lb + 2\epsilon \relerror < (1-\epsilon) \ub$ or $\lb=0$}
			\State $g \gets $ sample return that satisfies Assumption \ref{ass_one} (e.g., see Algorithm \ref{alg_discounted} in Appendix \ref{app_algorithms})
			\State $\Delta \gets g - \gbar$
			\State $\gbar \gets \gbar + \frac{\Delta}{j}$
			\State $M \gets M + \Delta (g - \gbar)$
			\State $\sigma \gets \sqrt{M/j}$
			\State $\triangleright$ Compute the confidence interval
			\If{$j \ge \floor{\beta^h}$}
			\State $h \gets h + 1$
			\State $\alpha \gets \floor{\beta^h}/\floor{\beta^{h-1}}$
			\State $x \gets -\alpha \log \cfrac{\delta(p-1)}{3 p h^p} $  
			\EndIf
			\State $c_j \gets \sigma \sqrt{\frac{2x}{j}} + \frac{3 \vmax x}{ j}$
			\State $\lb \gets \max(\lb, | \gbar | - c_j)$
			\State $\ub \gets \min(\ub, | \gbar | + c_j)$
			\State $\lbtrue \gets \max(\lbtrue, \gbar  - c_j)$
			\State $\ubtrue \gets \min(\ubtrue, \gbar  + c_j)$
			\If{$\frac{\ubtrue - \lbtrue}{2} \leq \epsilon \tau$}
			\Return $\frac{\ubtrue + \lbtrue}{2}$	
			\EndIf
			\State $j = j+1$
			\EndWhile
			\Return $\frac{\sign{\gbar}}{2} ((1+\epsilon)\lb + (1-\epsilon)\ub)$
		\end{algorithmic}
	\end{algorithm}

	\paragraph{Specifying the number of sampled states $\nsampledstates$.}
	For the number of required samples for the outer loop in Algorithm 1, we need enough samples to match the bound in Lemma \ref{lem_one}.
	\begin{align}
	\errstates = \sqrt{\frac{\log(4 \nexperiments/\delta) \cliperror^2}{2\nsampledstates}}
	\implies \nsampledstates = \frac{ \log(4 \nexperiments /\delta) \cliperror^2}{2\errstates^2}
	\end{align}
	$\nsampledstates$ is chosen as $\ceil{\frac{ \log(4 \nexperiments /\delta) \cliperror^2}{2\errstates^2}} \geq \errstates$ and thus we are being slightly conservative regarding the error to ensure correctness with high probability. We opt for a separate choice of $\errstates$ for this part of the bound,
	because it is completely separate from the other errors.
	This number $\errstates$ could be chosen slightly larger,
	to reduce the number of required sampled states to compare to,
	whereas $\errebstop$ might need to be smaller depending on the choice
	of $\cliperror$ and $\relerror$. Separating them explicitly
	can significantly reduce the $\nsampledstates$ in the outer loop,
	both improving time and storage, as well as later comparison time,
	without impacting the accuracy of the algorithm. 
	
	\paragraph{Satisfying Assumption \ref{ass_two}.}
	
	Our goal is to get an ($\errebstop,\delta,\relerror$)-approximation of $\vsampled(s_i)$, with a feasible number of samples.
	In many cases, it is difficult to make parametric assumptions about returns in reinforcement learning. 
	A simple strategy is to use a stopping rule for generating returns, based on general concentration inequalities---like Hoeffding's bound---that make few assumptions about the random variables. If we had a bit more information, however, such as the variance of the returns, we could obtain a tighter bound, using Bernstein's inequality and so reduce the number of required samples. We cannot know this variance a priori, but fortunately an empirical Bernstein bound has been developed  \citep{mnih2008empirical}. Using this bound,  \citet{mnih2008empirical} designed EBGStop, which incrementally estimates variance and significantly reduces the number of samples required to get high-confidence estimates.
	
	EBGStop can be used, without modification, given a mechanism to sample truncated returns that satisfy Assumption \ref{ass_one}. However, we generalize the algorithm to allow for our less restrictive condition $| \vsampledtrue(s_i) -  \vsampled(s_i) | \le \epsilon (| \vsampledtrue(s_i) | + \relerror)$, as opposed to the original algorithm which ensured $| \vsampledtrue(s_i) -  \vsampled(s_i) | \le \epsilon | \vsampledtrue(s_i) |$. When $\tau = 0$ in our algorithm, it reduces to the original; since this is a generalization on that algorithm, we continue to call it EBGStop. 
	This modification is important when $\vtrue(s_i) = 0$, since this would require $\vsampledtrue(s_i) = \vtrue(s_i)$ when $\tau = 0$. 
	For $\tau > 0$, once the accuracy is within $\tau$, the algorithm can stop. 
	The Algorithm is summarized in Algorithm \ref{alg_ebgstop}. The proof follows closely to the proof for EBGStop; we include it in Appendix \ref{app_proofs}.
	Algorithm 2 uses geometric sampling, like EBGStop, to improve sample efficiency. The idea is to avoid checking the stopping condition after every sample. Instead, for some $\beta > 1$, the condition is checked after $\beta^k$ samples; the next check occurs at $\beta^{k+1}$. This modification improves sample efficiency from a multiplicative factor of $\log(\frac{R}{\epsilon |\mu|})$ to $\log\log(\frac{R}{\epsilon |\mu|})$, where $R$ is the range of the random variables and $\mu$ is the mean. 
	
	\begin{lemma}\label{lemma_ebstop}
		Algorithm 2 returns an $\errebstop,\delta,\relerror$-approximation $\vsampled(s_i)$:
		\begin{align*}
		| \vsampledtrue(s_i) -  \vsampled(s_i) | &\le \epsilon (| \vsampledtrue(s_i) | + \relerror)
		\end{align*}
	\end{lemma}
	
	\begin{corollary}
		\label{corollary:1}
		For any $0 < \errstates$ and $0 < \erralg < 1$,
		Algorithm \ref{alg_error} returns an $\epsilon,\delta$-accurate approximation: with probability at least $1-\delta$,
		\begin{equation*}
		\left| \clipf(\vapprox,\vtrue)  - \frac{1}{\nsampledstates} \sum_{i = 1}^\nsampledstates \clipfs(\vapprox(s_i),  \vsampled(s_i)) \right|  \le \epsilon
		\end{equation*}
		where
		\begin{equation}
		\epsilon = \errstates + 2 (1+\cliperror) \erralg
		\end{equation}
	\end{corollary}
	Algorithm \ref{alg_error} uses this theorem, for a given desired level of accuracy $\epsilon$. 
	To obtain this level of accuracy, $\epsilon_m = \epsilon/2$ and $\erralg$ given to Algorithm \ref{alg_ebgstop} is set to ensure $2 (1+\cliperror) \erralg = \epsilon$. 
	
	%
	%
	%

	\begin{figure*}
		\centering
		\begin{subfigure}[b]{0.490\textwidth}
			\centering
			\includegraphics[width=\textwidth]{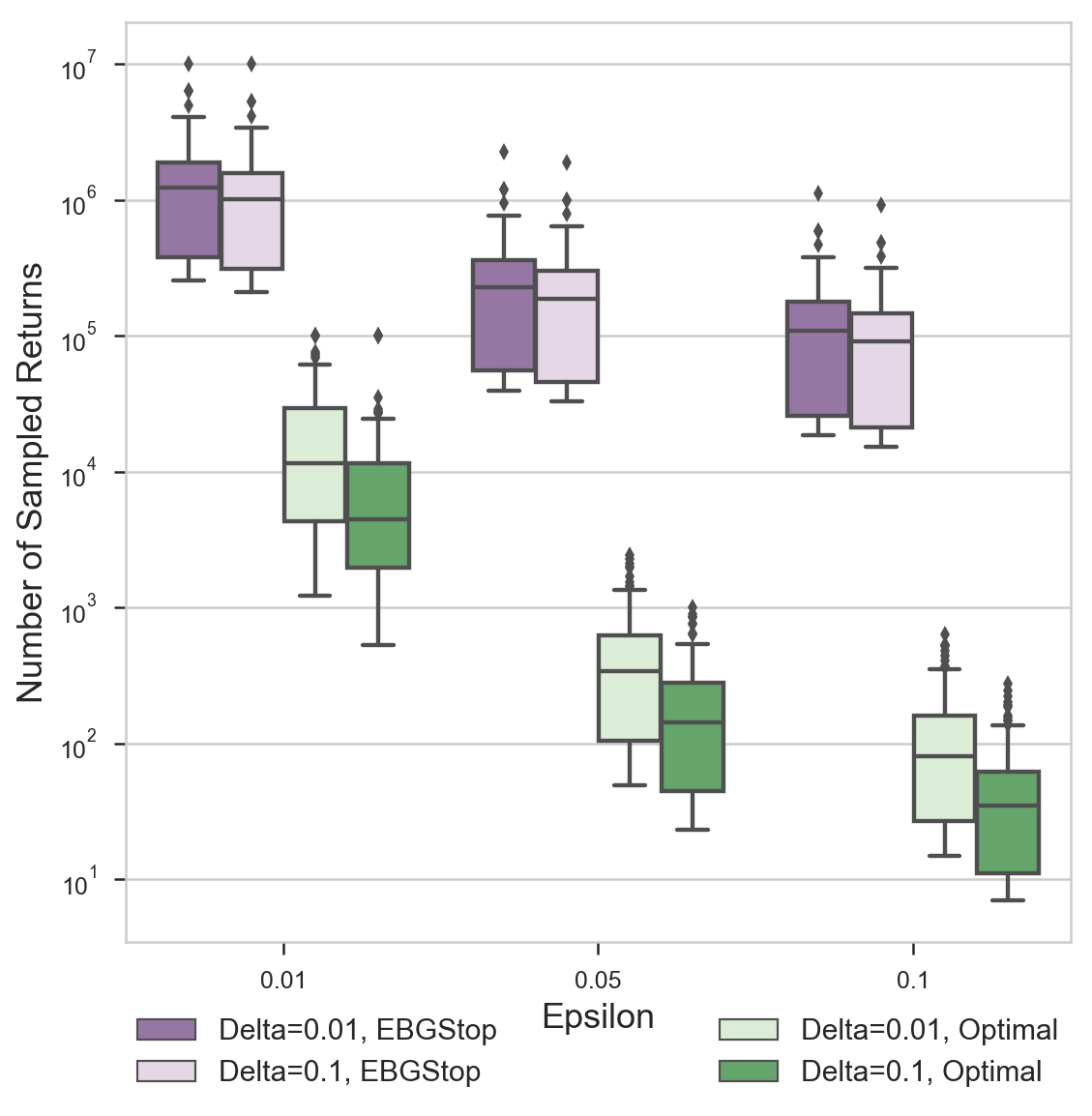}
			\caption{Puddle World}
			\label{fig:PW_A23}
		\end{subfigure}
		~
		\begin{subfigure}[b]{0.490\textwidth}
			\centering
			\includegraphics[width=\textwidth]{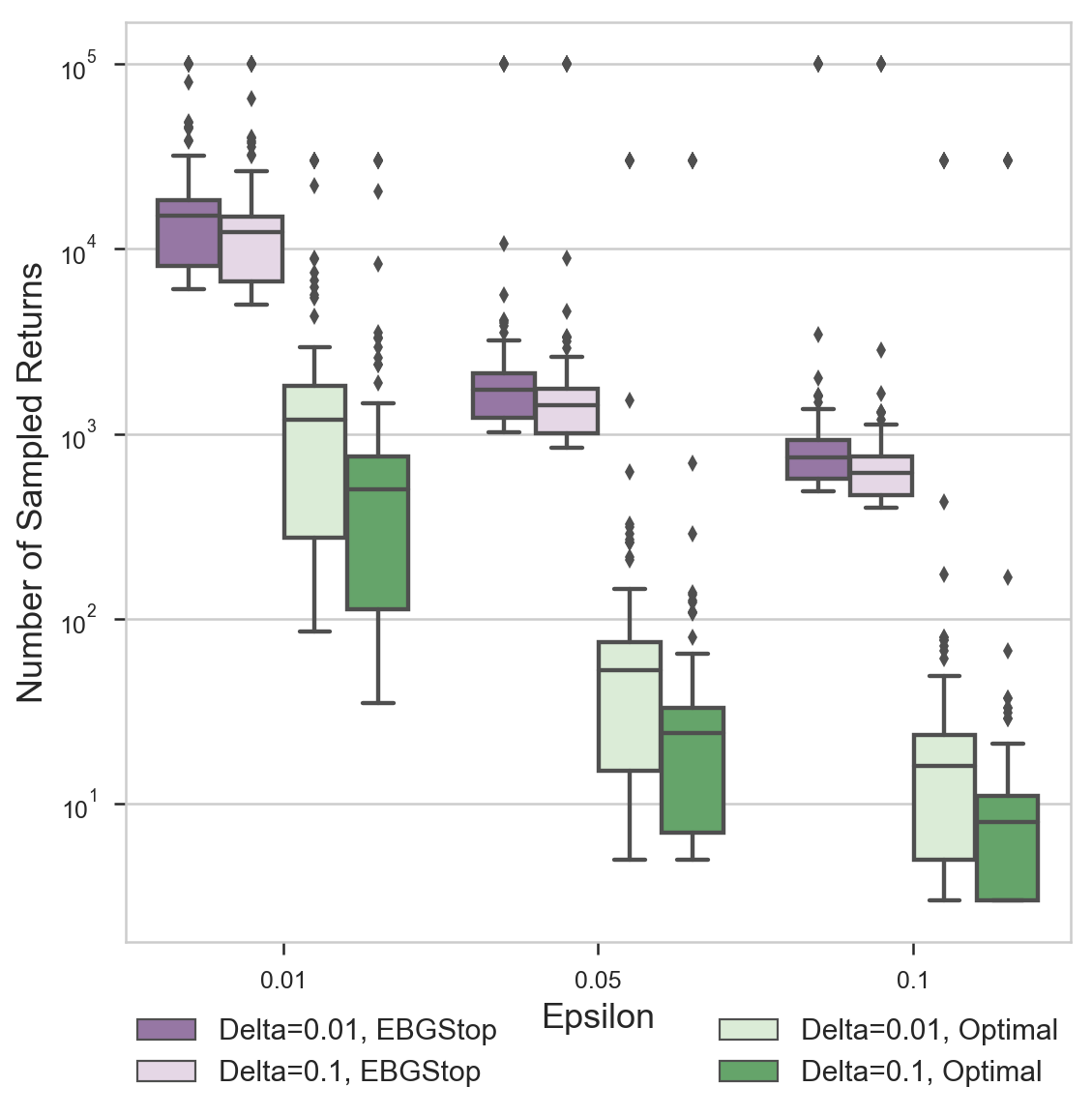}
			\caption{Mountain Car}
			\label{fig:MC_06}
		\end{subfigure}
		
		\caption{The number of sampled returns to obtain high-confidence estimates $\vsampled$ of the true values $\vtrue$. 
			The y-axis is logarithmic, with many more samples used for smaller $\epsilon$. The box-plot similarly are logarithmic, and so are actually much larger for smaller $\epsilon$ than larger $\epsilon$. The accuracy $\epsilon$ has a much larger effect on the number of sampled returns that are required, than the probability $\delta$. An additional point of interest is that there are a few states that required significantly more samples for the returns, indicated by the outliers depicted as individual points. } \label{fig:PW_MC_exp_1}
		
	\end{figure*}
	
	\section{EXPERIMENTS ON BENCHMARK PROBLEMS}
	
	We investigate the required number of samples to get with a level of accuracy, for different probability levels.
	We report this for two continuous-state benchmark problems---Mountain Car and Puddle World---which have previously
	been used to compare policy evaluation algorithms. Our goal is to (a) demonstrate how this framework can be used to obtain high-confidence estimates of accuracy and (b) provide some insight into how many samples are needed,
	even for simple reinforcement learning benchmark domains. 
	
	We report the number of returns sampled by Algorithm 2, averaged across several states.
	The domains, Mountain Car and Puddle World, are as specified in the policy evaluation experiments by \citet{pan2017accelerated}. For Mountain Car, we use the energy pumping policy, with $60\%$ random action selection for the three actions. For Puddle World, we used a uniform random policy for the four actions.
	They are both episodic tasks, with a maximum absolute value of $\vmax = 100$ and $\vmax = 8000$ respectively. 
	The variance in Puddle World is particularly high, as it has regions with high-variance, high-magnitude rewards. 
	We sampled $\nsampledstates = 100$ states uniformly across the state-space, to provide some insight into the variability of the number of returns 
	sampled across the state-space. 
	We tested $\epsilon \in \{0.01,0.05,0.1\}$ and $\delta \in \{0.01, 0.1\}$, and set $\tau = 1.0$. 
	We focus here on how many returns need to be sampled, rather than the trajectory length, and so do not use $\cliperror$ nor explicitly compute clipped errors $\clipfs$. 
	
	The results indicate that EBGStop requires a large number of samples, particularly in Puddle World. 
	Figure \ref{fig:MC_06} for Mountain Car and Figure \ref{fig:PW_A23} both indicate that decreasing $\epsilon$ from $0.05$ to $0.01$, to enable higher-accuracy estimates of value function error, causes an exponential increase in the required number of samples, an increase of $10^3$ to $10^4$ for Mountain Car and $10^5$ to $10^6$ for Puddle World. An accuracy level of $0.01$, which corresponds to difference of 1\% for clipped errors, is a typical choice for policy evaluation experiments, yet requires an inordinate number of samples, particularly in Puddle World. 
	
	We further investigated lower bounds on the required number of samples. Though EBGStop is a state-of-the-art stopping algorithm,
	to ensure high-confidence bounds for any distribution with bounded mean and variance, it collects more samples than is actually required.
	To assess its efficiency gap, we also include an idealistic approach to computing the confidence intervals, using repeated subsamples computed from the simulator. By obtaining many, many estimates of the sample average, using $t$ samples of the truncated return, we can estimate the actual variability of the sample average. We provide additional details in Appendix \ref{app_algorithms}. Such a method to compute the confidence interval is not a viable algorithm to reduce the number of samples generated. Rather, the goal here is to report a lower bound on the number of samples required, for comparison and to motivate the amount the sampling algorithm could be improved. The number of samples generated by EBGStop is typically between 10 to 100 times more than the optimal number of samples, which indicates that there is much room to improve sample efficiency.

	\section{CONCLUSION}
	
	In this work, we present the first principled approach to obtain high-confidence error estimates of learned value functions. 
	Our strategy is focused on the setting tackled by reinforcement learning empiricists, comparing value function-learning algorithms. In this context, accuracy of value estimates, for multiple algorithms, need to be computed repeatedly, every few steps with increasing data given to the learning algorithms. 
	We provide a general framework for such a setting, where we store estimates of true value functions using samples of truncated returns.
	The framework for estimating true values for comparison is intentionally generic, to enable any (sample-efficient) stopping algorithm to be used.
	We propose one solution, which uses empirical Bernstein bounds, to significantly reduce the required number of samples over other concentration inequalities, such as Hoeffding's bound.    
	
	This paper highlights several open challenges. As demonstrated in the experiments, there is a large gap between the actual required number of samples and that provided by the algorithm using an empirical Bernstein stopping-rule. For some simulators, this overestimate could result in a prohibitively large number of samples. Although this is a problem more generally faced by the sampling literature, it is particularly exacerbated in reinforcement learning where the variability across states and returns can be high, with large maximum values. An important avenue, then, is to develop more sample-efficient sampling algorithms to make high-confidence error estimates feasible for a broader range of settings in reinforcement learning.
	
	Another open challenge is to address how to sample states $\{s_1, \ldots, s_\nsampledstates\}$. This paper is agnostic to how these states are obtained. However, it is not always straightforward to sample these from a desired distribution. Some choices are simple, such as randomly selecting these across the state space. For other cases, it is more complicated, such as sampling these from the stationary distribution of the behaviour policy, $\dmu$. The typical strategy is to run $\mu$ for a burn-in period, so that afterwards it is more likely for states to be sampled from the stationary distribution. The theoretical effectiveness of this strategy, however, is not yet well-understood. There has been work estimating empirical mixing times \citep{hsu2015mixing} and some work bounding the number of samples required for burn-in \citep{paulin2015concentration}. Nonetheless, it remains an important open question on how to adapt these results for the general reinforcement learning setting. 
	
	One goal of this paper has been to highlight an open problem that has largely been ignored by reinforcement learning empiricists. We hope for this framework
	to stimulate further work in high-confidence estimates of value function accuracy.

	\bibliographystyle{unsrtnat}
	\bibliography{paper}
	
	\clearpage
	
	\appendix
	
	\section{SUPPLEMENTARY LEMMAS}\label{app_proofs}
	
	\begin{lemma}\label{lem_triangle}
		The clipped error $\clipfs$ satisfies the triangle inequality. 
	\end{lemma}
	\begin{proof}
		This results follows because $|x|_\cliperror = |x - y + y|_\cliperror \le | x - y|_\cliperror + |y|_\cliperror$,
		still holds under clipping. To see why, consider the following. If either $| x - y|_\cliperror $ or $|y|_\cliperror$ are clipped to $\cliperror$,
		then clearly the sum is larger than $|x|_\cliperror$. Otherwise, if only $|x|_\cliperror$ is clipped to $\cliperror$,
		then it can only have been strictly decreased and again the inequality must hold. Once we have this inequality,
		we can use the fact that $|x| \le | x - y| + |y|$ and $|y| \le | x - y| + |x|$ to get the $|x| - |y| \le |x-y|$ and $|y| - |x| \le |x-y|$.
	\end{proof}
	
	\textbf{Lemma \ref{lemma_ebstop}.}
	For a state $s \in \States$, Algorithm 2 returns an $\errebstop,\delta,\relerror$-approximation $\vsampled(s)$:
	\begin{equation}
	| \vsampledtrue(s) -  \vsampled(s) | \le \epsilon | \vsampledtrue(s) | + \epsilon \relerror  \tag{\ref{eq_asstwo}}
	\end{equation}
	\begin{proof}
		We follow a similar argument to  \citet[Section 3.1]{mnih2008empirical}. The empirical Bernstein bound \citep{audibert2007tuning} states that,
		for a sample average $\gbar_t = \frac{1}{\ntrajectories} \sum_{j=1}^\ntrajectories g_j$ of $t$ unbiased samples $g_j$ 
		\begin{equation*}
		| \vsampledtrue(s) -  \gbar_t | \le c_\ntrajectories 
		\end{equation*}
		where
		\begin{align}
		c_\ntrajectories &= \stda  \sqrt{\frac{2 \log(3/\delta)}{\ntrajectories}} +  3  \log(3/\delta)\frac{\vmax}{\ntrajectories} \\
		\stda &= \sqrt{\frac{1}{\ntrajectories} \sum_{j= 1}^\ntrajectories (g_{j} - \vsampled(s) )^2} 
		\end{align}
		Algorithm 2 estimates lower and upper bounds, based on this concentration inequality, guaranteeing that the absolute value of the true value is between these bounds with probability at least $1-\delta$. 
		Algorithm 2 terminates when either of the following cases are satisfied:
		
		\textbf{[Case : 1]}
		$(1+\epsilon)\lb + 2\epsilon \relerror \geq (1-\epsilon)\ub$ and returns $\vsampled = \frac{\sign{\gbar_t}}{2} ((1+\epsilon)\lb + (1-\epsilon)\ub)$. 
		
		\textbf{[Case : 2]}
		$\frac{\ubtrue - \lbtrue}{2} \leq \epsilon \relerror$ and if so, the algorithm outputs $\vsampled = \frac{\ubtrue + \lbtrue}{2}$. This second case is for the setting where $\vsampledtrue(s) = 0$, or very near zero, meaning it would not terminate in Case 1. The relative error will remain high, even though $\vsampled(s)$ is sufficiently close to $\vsampledtrue(s)$ to satisfy \eqref{eq_asstwo} because of $\tau > 0$. 
		
		We show that for both cases, \eqref{eq_asstwo} is satisfied. 
		We begin with the proof for Case 1. Assume the algorithm terminated, according to the condition in Case 1. For all $j \in \{1, \ldots, \ntrajectories\}$, $c_j > 0$
		and $\ub > 0$ since $\ub = \min_j(|\gbar_j| + c_j)$.
		Upon termination, we have with probability $1-\delta$,
		\begin{align*}
		|\vsampled(s)| 
		&= \frac{(1+\epsilon)\lb + (1-\epsilon)\ub}{2} \\ 
		&\leq \frac{(1+\epsilon)\lb + (1+\epsilon)\lb + 2\epsilon \relerror}{2} \\ 
		&= (1+\epsilon)\lb + \epsilon \relerror \\ 
		&\leq (1+\epsilon) | \vsampledtrue(s) | + \epsilon \relerror.
		\end{align*} 
		Similarly, 
		\begin{align*}
		|\vsampled (s) | 
		&= \frac{(1+\epsilon)\lb + (1-\epsilon)\ub}{2}\\
		&\geq \frac{(1-\epsilon)\ub + (1-\epsilon)\ub + 2\epsilon \relerror}{2} \\
		&\geq (1-\epsilon) | \vsampledtrue(s) | + \epsilon \relerror
		\end{align*}
		Combining these two inequalities gives 
		\begin{equation}
		\big||\vsampled(s)| - |\vsampledtrue(s)| \big| \leq \epsilon |\vsampledtrue(s)| + \epsilon \relerror \label{eq_abs}
		.
		\end{equation}
		%
		When termination occurs under Case, we know $\lb > 0$, and so 
		$|\gbar_t| \ge c_t \geq |\gbar_t - \vsampledtrue(s)|$. This is because $|\gbar_t| - c_t$ must have increased the lower bound, to allow termination.  This inequality, $|\gbar_t| \geq |\gbar_t - \vsampledtrue(s)|$ is only possible if $\vsampledtrue(s)$ is of the same sign as $\gbar_t$. This gives that $\sign{\vsampled(s)}$ = $\sign{\gbar_t}$ = $\sign{\vsampledtrue(s)}$.
		Because the signs match,  $\big||\vsampled(s)| - |\vsampledtrue(s)| \big|  = |\vsampled(s) - \vsampledtrue(s)| $,
		and so the result follows from Equation \eqref{eq_abs}.
		
		For Case 2, the interval $[\lbtrue, \ubtrue]$ represents the confidence interval from the IID samples that contains the true mean $\vsampledtrue$. The terminating condition is $\frac{\ubtrue - \lbtrue}{2} \leq \epsilon \relerror$. For $\vsampled(s) = \frac{\ubtrue + \lbtrue}{2}$, this gives $\ub - \vsampled(s) = \frac{\ubtrue - \lbtrue}{2} \leq \epsilon \relerror$ and 
		$\vsampled(s) - \lb = \frac{\ubtrue - \lbtrue}{2} \leq \epsilon \relerror$. Upon termination, therefore, we have
		$\epsilon \relerror \geq \ubtrue - \vsampled \geq \vsampledtrue - \vsampled$ and $\epsilon \relerror \geq \vsampled - \lbtrue \geq \vsampled - \vsampledtrue$. Thus, $|\vsampled - \vsampledtrue| \leq \epsilon \relerror \leq \epsilon | \vsampledtrue | + \epsilon \relerror$.
%
		%
%
	\end{proof}

	\section{HIGH CONFIDENCE BOUNDS FOR CLIPPED MAVE AND MSVE} \label{app_clipped_loss}
	
	If one desires to use non-percentage losses, corresponding high-confidence sample complexity bounds are derivable. 
	In this section, we will extend our analysis to the clipped Mean Absolue Value Error (CMAVE) and clipped Mean Squared Value Error (CMSVE). 
	
	These are defined as follows:
	\begin{align*}
	\text{CMAVE}(\vapprox, \vsampled) &\defeq \E \left[  \min(c,|\vapprox(s_i) - \vsampled(s_i)|) \right] \\
	\text{CMSVE}(\vapprox, \vsampled) &\defeq \E \left[ \min(c,(\vapprox(s_i) - \vsampled(s_i))^2) \right]
	\end{align*} 
	Along with their empirical approximations:
	\begin{align*}
	\text{CMAVE}(\vapprox, \vsampled) &\approx \frac{1}{\nsampledstates} \sum_{i=1}^{m}  \min(c,|\vapprox(s_i) - \vsampled(s_i)|) \\
	\text{CMSVE}(\vapprox, \vsampled) &\approx \frac{1}{\nsampledstates} \sum_{i=1}^{m} \min(c,(\vapprox(s_i) - \vsampled(s_i))^2)
	\end{align*} 
	In proving the sample complexity results, we use some of the ideas used in proving Theorem \ref{thm_main}.
	Since both CMAVE and CMSVE are non-percentage losses and do not require a division by the value function, the analysis is greatly simplified. In fact, they no longer require Assumption \ref{ass_two} and, hence, remove the need for EBGStop-like algorithms (1 and 2) presented in Section 4 (which deal with relative errors). 
	Instead of using EBGStop to provide an estimate of the value function, we can simply compute the appropriate number of truncated rollouts (sampled returns) to achieve an estimate of the desired accuracy. These sample complexity numbers are provided in the following analysis.
	

	
	
	\subsection{SAMPLE COMPLEXITY ANALYSIS OF CLIPPED MAVE}
    In this section, we will use $\CMAVEtrue(\vapprox, \vsampled)$ to refer CMAVE($\vapprox, \vsampled$). Also, the following definitions will be necessary for our analysis:
	\begin{align*}
	\CMAVE(\vapprox(s_i), \vsampled(s_i)) &\defeq  \min(c,|\vapprox(s_i) - \vsampled(s_i)|) \\
	\CMAVEapprox(\vapprox, \vsampled) &\defeq \frac{1}{\nsampledstates} \sum_{i=1}^{m} \min(c,|\vapprox(s_i) - \vsampled(s_i)|) \\ 
	\CMAVEtrue(\vapprox, \vsampled) &\defeq \E[\CMAVEapprox(\vapprox, \vsampled)] \\
	\end{align*}
	We also define similar quantities replacing $\vsampled$ with $\vtrue$ in the above definitions. Below, we present the sample complexity bound for CMAVE.
	
	\newcommand{\std}{\bar{\sigma}}
	
\begin{theorem} \label{thm_CMAVE}
	Let $\{ s_1, \ldots, s_\nsampledstates \}$ be states sampled I.I.D according to $\weighting$ and that the number of rollouts for each state be $n$. Let $\std_i$ be the standard deviation of the rollouts for state $i$.
	
	With probability at least $1-\delta$ the following bound for clipped MAVE holds:
	\begin{equation}
	\left| \CMAVEtrue(\vapprox,\vtrue)  -  \CMAVEapprox(\vapprox,\vsampled) \right| \le \sqrt{\frac{\log(4 \nexperiments/\delta) \cliperror^2}{2\nsampledstates}} + \zeta.
	\end{equation}
	for $\zeta = 3 \rmax \left( \cfrac{1-\gamma^{\lenrollout}}{1-\gamma} \right) \cfrac{\log(6 \nsampledstates/ \delta)}{n} + \frac{\sum_{i=1}^{m}\std_i}{m} \sqrt{\cfrac{2\log(6 \nsampledstates/ \delta)}{n}} + \rmax \cfrac{\gamma^{\lenrollout}}{1-\gamma}$. 
\end{theorem}

	\begin{proof}
		Similar to Theorem \ref{thm_main}, we start by bounding $\left| \CMAVEtrue(\vapprox,\vtrue)   - \CMAVEapprox(\vapprox, \vsampled)  \right|$: 
		\begin{align*}
		\left| \CMAVEtrue(\vapprox,\vtrue)   - \CMAVEapprox(\vapprox, \vsampled)  \right|
		&\leq 
		\left| \CMAVEtrue(\vapprox,\vtrue)   - \CMAVEapprox(\vapprox, \vtrue)  \right| \\
		&+ \left| \CMAVEapprox(\vapprox,\vtrue)   - \CMAVEapprox(\vapprox, \vsampled)  \right|
		\end{align*}
		The first term is bounded by Hoeffding's inequality in Lemma \ref{lem_one} with probability at least $1-\delta/2$ which gives:
		\begin{align*}
		\left| \CMAVEtrue(\vapprox,\vtrue)   - \CMAVEapprox(\vapprox, \vtrue)  \right| \leq \sqrt{\frac{\log(4 \nexperiments/\delta) \cliperror^2}{2\nsampledstates}}
		\end{align*}
		
		The second term is bounded in the following way:
		\begin{align*}
		&\left| \CMAVEapprox(\vapprox,\vtrue)  - \CMAVEapprox(\vapprox, \vsampled)  \right|\\
		&\leq  \frac{1}{\nsampledstates} \sum_{i = 1}^\nsampledstates \left|  \CMAVE( \vapprox(s_i), \vtrue(s_i) )  - \CMAVE( \vapprox(s_i), \vsampled(s_i) )  \right|
		\end{align*}
		We can bound each one of these terms as follows:
		\begin{multline}
		\left|  \CMAVE( \vapprox(s_i), \vtrue(s_i) )  - \CMAVE( \vapprox(s_i), \vsampled(s_i) )  \right| \\
		= \left| \min(c,|\vapprox(s_i) - \vtrue(s_i)|) - \min(c,|\vapprox(s_i) - \vsampled(s_i)|) \right| \\ 
		\leq \max \bigg( \bigg| \min \big(c,|\vapprox(s_i) - \vsampled(s_i)| + |\vsampled(s_i) - \vtrue(s_i)| \big) - \\ \min \big(c,|\vapprox(s_i) - \vsampled(s_i)| \big) \bigg|, \bigg|
		\min \big(c,|\vapprox(s_i) - \vtrue(s_i)| \big) - \\ \min \big(c,|\vapprox(s_i) - \vtrue(s_i)| + |\vsampled(s_i) - \vtrue(s_i)| \big) \bigg| \bigg) \\
		\leq |\vsampled(s_i) - \vtrue(s_i)| \leq |\vsampled(s_i) - \vsampledtrue(s_i)| + |\vsampledtrue(s_i) - \vtrue(s_i)| \\ \leq \zeta.
		\end{multline}
		
		Now, we need to find an expression of $\zeta$. The term $|\vsampled(s_i) - \vsampledtrue(s_i)|$ can be bounded using the empirical bernstein inequality for random variables with range: $\rmax \left( \cfrac{1-\gamma^{\lenrollout}}{1-\gamma} \right)$. This leads us to a bound: $|\vsampled(s_i) - \vsampledtrue(s_i)| \leq 3 \rmax \left( \cfrac{1-\gamma^{\lenrollout}}{1-\gamma} \right) \cfrac{\log(6 \nsampledstates/ \delta)}{n} + \std_i \sqrt{\cfrac{2\log(6 \nsampledstates/ \delta)}{n}}$. The second term can be bounded based on the proof of Lemma \ref{lem_discounted}: $|\vsampledtrue(s_i) - \vtrue(s_i)| \leq \rmax \cfrac{\gamma^{\lenrollout}}{1-\gamma}$. This gives us $\zeta = 3 \rmax \left( \cfrac{1-\gamma^{\lenrollout}}{1-\gamma} \right) \cfrac{\log(6 \nsampledstates/ \delta)}{n} + \frac{\sum_{i=1}^{m}\std_i}{m} \sqrt{\cfrac{2\log(6 \nsampledstates/ \delta)}{n}} + \rmax \cfrac{\gamma^{\lenrollout}}{1-\gamma}$. We finish the proof by pointing out that due to using hoeffding bound twice with error probability of atmost $\delta/2$ and due to the union bound (to ensure that the bound holds for all $m$ states), the probability that the final bound holds is with at least $1- \delta$.
	\end{proof}
	
	\subsection{SAMPLE COMPLEXITY ANALYSIS OF CLIPPED MSVE}
	
	In this section, we will use $\CMSVEtrue(\vapprox, \vsampled)$ to refer CMSVE($\vapprox, \vsampled$). Also, the following definitions will be necessary for our analysis:
	\begin{align*}
	\CMSVE(\vapprox(s_i), \vsampled(s_i)) &\defeq \min(c,(\vapprox(s_i) - \vsampled(s_i))^2) \\
	\CMSVEapprox(\vapprox, \vsampled) &\defeq \frac{1}{\nsampledstates} \sum_{i=1}^{m} \min(c,(\vapprox(s_i) - \vsampled(s_i))^2) \\ 
	\CMSVEtrue(\vapprox, \vsampled) &\defeq \E[\CMSVEapprox(\vapprox, \vsampled)] \\
	\end{align*}
	Similarly, in the above definitions, $\vtrue$ can be used instead of $\vsampled$. Below, we present the sample complexity bound for CMSVE.
	\begin{theorem} \label{thm_CMSVE}
		Let $\{ s_1, \ldots, s_\nsampledstates \}$ be states sampled I.I.D according to $\weighting$ and that the number of rollouts for each state be $n$. Let $\std_i$ be the standard deviation of the rollouts for state $i$.
		
		With probability at least $1-\delta$ the following bound for clipped MSVE holds:
		\begin{equation}
		\left| \CMSVEtrue(\vapprox,\vtrue)  -  \CMSVEapprox(\vapprox,\vsampled) \right| \le \sqrt{\frac{\log(4 \nexperiments/\delta) \cliperror^2}{2\nsampledstates}} + \zeta
		\end{equation}
		for $\zeta = 3 \rmax^2 \left( \cfrac{1-\gamma^{\lenrollout}}{1-\gamma} \right)^2 \cfrac{\log(6 \nsampledstates/ \delta)}{n} + \frac{\sum_{i=1}^{m}\std_i}{m} \sqrt{\cfrac{2\log(6 \nsampledstates/ \delta)}{n}} + \rmax^2 \left( \cfrac{\gamma^{\lenrollout}}{1-\gamma} \right)^2$
		
	\end{theorem}
	
	\begin{proof}
		Similar to Theorem \ref{thm_main}, we start by bounding $\left| \CMSVEtrue(\vapprox,\vtrue)   - \CMSVEapprox(\vapprox, \vsampled)  \right|$: 
		\begin{align*}
		\left| \CMSVEtrue(\vapprox,\vtrue)   - \CMSVEapprox(\vapprox, \vsampled)  \right|
		&\leq 
		\left| \CMSVEtrue(\vapprox,\vtrue)   - \CMSVEapprox(\vapprox, \vtrue)  \right| \\
		&+ \left| \CMSVEapprox(\vapprox,\vtrue)   - \CMSVEapprox(\vapprox, \vsampled)  \right|
		\end{align*}
		
		The first term is bounded by Hoeffding's inequality in Lemma \ref{lem_one} with probability atleast $1-\delta/2$ which gives:
		\begin{align*}
		\left| \CMAVEtrue(\vapprox,\vtrue)   - \CMAVEapprox(\vapprox, \vtrue)  \right| \leq \sqrt{\frac{\log(4 \nexperiments/\delta) \cliperror^2}{2\nsampledstates}}
		\end{align*}
		
		The second term is bounded in the following way:
		\begin{align*}
		&\left| \CMSVEapprox(\vapprox,\vtrue)  - \CMSVEapprox(\vapprox, \vsampled)  \right|\\
		&\leq  \frac{1}{\nsampledstates} \sum_{i = 1}^\nsampledstates \left|  \CMSVE( \vapprox(s_i), \vtrue(s_i) )  - \CMSVE( \vapprox(s_i), \vsampled(s_i) )  \right|
		\end{align*}
		
		We can bound each one of these terms as follows:
		\begin{multline}
		\left|  \CMSVE( \vapprox(s_i), \vtrue(s_i) )  - \CMSVE( \vapprox(s_i), \vsampled(s_i) )  \right| \\
		= \left| \min(c,(\vapprox(s_i) - \vtrue(s_i))^2) - \min(c,(\vapprox(s_i) - \vsampled(s_i))^2) \right| \\ 
		\leq \max \bigg( \bigg| \min \big(c,|\vapprox(s_i) - \vsampled(s_i)|^2 + |\vsampled(s_i) - \vtrue(s_i)|^2 \big) - \\ \min \big(c,|\vapprox(s_i) - \vsampled(s_i)|^2 \big) \bigg|, \bigg|
		\min \big(c,|\vapprox(s_i) - \vtrue(s_i)|^2 \big) - \\ \min \big(c,|\vapprox(s_i) - \vtrue(s_i)|^2 + |\vsampled(s_i) - \vtrue(s_i)|^2 \big) \bigg| \bigg) \\
		\leq |\vsampled(s_i) - \vsampledtrue(s_i)|^2 + |\vsampledtrue(s_i) - \vtrue(s_i)|^2 \leq \zeta
		\end{multline}
		
		The first inequality is due to $|a - b|^2 \leq (|a-c| + |c-b|)^2$ and this implies $|a - b|^2 \leq |a-c|^2 + |c-b|^2 \leq (|a-c| + |c-b|)^2$. The range of $\vsampled$ is $\rmax^2 \left( \cfrac{1-\gamma^{\lenrollout}}{1-\gamma} \right)^2$. Now, using the empirical bernstein inequality, we can follow the proof technique in Theorem \ref{thm_CMAVE} to show that $\zeta = 3 \rmax^2 \left( \cfrac{1-\gamma^{\lenrollout}}{1-\gamma} \right)^2 \cfrac{\log(6 \nsampledstates/ \delta)}{n} + \frac{\sum_{i=1}^{m}\std_i}{m} \sqrt{\cfrac{2\log(6 \nsampledstates/ \delta)}{n}} + \rmax^2 \left( \cfrac{\gamma^{\lenrollout}}{1-\gamma} \right)^2$. 
		
		We finish the proof similarly to conclude that the final bound holds with probability atleast $1-\delta$ due to the application of hoeffding's bound twice with error probability of atmost $\delta/2$ and due to the union bound.
	\end{proof}

%
%
%
\section{SAMPLE COMPLEXITY ANALYSIS OF UNCLIPPED LOSSES}
Sometimes, one may prefer to consider unclipped losses. Here, we will present sample complexity bounds for the Mean Absolute Value Error (MAVE) and the Mean Squared Value Error (MSVE). 
To derive meaningful bounds for unbounded random variables, we need to impose other assumptions. There are various options but we choose to explore one: using sub-exponential random variables.
With this assumption, the proof techniques mostly follow from those in Section \ref{app_clipped_loss} of the appendix. Below we briefly discuss sub-exponential random variables and illustrate how one can derive the corresponding concentration bounds. 

\subsection{SUB-EXPONENTIAL CONCENTRATION ANALYSIS}
It is well known that for unbounded random variables, finite high probability bounds are not derivable unless it is possible to assume a bound on the moment generating function. One way to derive a meaningful bound is to assume that the tails of the random variable's distribution decay exponentially. 
If we know the tail decay like a Gaussian distribution, sub-gaussianity is a common assumption. A weaker assumption is sub-exponentiality, which only requires that the moment generating function exists. The Laplace and exponential distributions are two such common fat-tailed distributions.
In this section, we will assume that the loss random variable is sub-exponential and derive finite sample complexity bounds. For completeness, we provide the necessary definitions.

\begin{definition} \label{def:subgauss}
	A sub-gaussian random variable $X$ with mean $\mu = \E[X]$ and parameters $\sigma \geq 0$ has the following bound on its moment generating function (MGF):
	\begin{equation}
	\E[e^{\lambda(X-\mu)}] \leq e^{\frac{\sigma^2 \lambda^2}{2}} \hspace{15pt} \forall \lambda \in \RR \label{eq:subgauss}
	\end{equation} 
\end{definition}

\begin{definition} \label{def:subexp}
	A sub-exponential random variable $X$ with mean $\mu = \E[X]$ and parameters $\alpha, \beta \geq 0$ has the following bound on its moment generating function:
	\begin{equation}
	\E[e^{\lambda(X-\mu)}] \leq e^{\frac{\alpha^2 \lambda^2}{2}} \hspace{15pt} \forall |\lambda| \leq \beta \label{eq:subexp}
	\end{equation} 
\end{definition}

Note that all sub-gaussian RVs are sub-exponential with $\alpha = \sigma$ and $\beta = \infty$, but not all sub-exponential RVs are sub-gaussian. For example, the gaussian distribution is a sub-exponential RV with $\alpha$ being the standard deviation and $\beta = \infty$. Thus, if the loss is known to be sub-gaussian, one can still use the sub-exponential concentration bound. Below, in Theorem \ref{thm:subexpconcentration}, we present a concentration bound for sub-exponential random variables.

\begin{theorem} \label{thm:subexpconcentration}
	If $X_i$ are I.I.D sub-exponential RVs with parameters ($\alpha$,$\beta$) as defined in Definition \ref{def:subexp}, then the following concentration bound holds:
	\begin{align*}
	Pr\left[ \left|\frac{1}{n} \sum^{n}_{i=1} X_i - \mu \right| \geq t \right] &\leq 2e^{-\frac{n t^2}{2 \alpha^2}} \hspace{13pt} \text{for} \hspace{5pt} 0 < t < \alpha^2 \beta \\
	Pr \left[ \left|\frac{1}{n}\sum^{n}_{i=1} X_i - \mu \right| \geq t \right] &\leq 2e^{-\frac{n t \beta}{2}} \hspace{13pt} \text{for} \hspace{5pt} t > \alpha^2 \beta 
	\end{align*}
\end{theorem}
\begin{proof}
	For $\lambda \geq 0$, 
	\begin{align*}
	Pr \left[\frac{\sum^{n}_{i=1} X_i}{n}-\mu \geq t \right] &= Pr \left[e^{\lambda \left( \frac{\sum^{n}_{i=1} X_i}{n} - \mu \right)} \geq e^{\lambda t} \right] \\ &\leq \cfrac{E\left[e^{\lambda \left( \frac{\sum^{n}_{i=1} X_i}{n} - \mu \right)} \right]}{e^{\lambda t}} \\ &= \cfrac{E\left[e^{\lambda \left( \frac{\sum^{n}_{i=1} (X_i - \mu)}{n} \right)} \right]}{\prod_{i=1}^{n} e^{\frac{\lambda t}{n}}} \\ &= \prod_{i=1}^{n} \left( \cfrac{E\left[e^{\lambda \left( \frac{(X_i - \mu)}{n} \right)} \right]}{e^{\frac{\lambda t}{n}}} \right). 
	\end{align*}
	We used Markov's inequality and the last equality is due to the independence of $X_i$. Now, we can bound the moment generating function of each $X_i$ using the sub-exponential RV's property (Definition \ref{def:subexp}). 
	
	$\therefore \cfrac{E[e^{\lambda(X_i - \mu)/n}]}{e^{\lambda t/n}} \leq e^{\frac{0.5\alpha^2 \lambda^2 - \lambda t}{n}}$
	
	Optimizing over $\lambda$ will result in the tightest bound possible. The minimum of $\frac{0.5 \alpha^2 \lambda^2 - \lambda t}{n}$ is reached at $\lambda = \frac{t}{\alpha^2}$. Replacing $\lambda$, we arrive at the expression: $-\frac{t^2}{2\alpha^2}$. By definition $\lambda < \beta$, which results in $t < \alpha^2 \beta$. As a result, the following bound for $t \in (0,\alpha^2 \beta)$ has to hold:
	
	$\cfrac{E[e^{\lambda(X_i - \mu)/n}]}{e^{\lambda t/n}} \leq e^{-\frac{t^2}{2 \alpha^2}}$
	
	$\therefore \prod_{i=1}^{n} \left( \cfrac{E\left[e^{\lambda \left( \frac{(X_i - \mu)}{n} \right)} \right]}{e^{\frac{\lambda t}{n}}} \right) \leq \prod_{i=1}^{n} e^{-\frac{t^2}{2 \alpha^2}} = e^{-\frac{nt^2}{2 \alpha^2}}$
	
	The same argument follows for the lower tail and by the union bound we conclude that $Pr \left[ \left|\frac{\sum^{n}_{i=1} X_i}{n}-\mu \right| \geq t \right] \leq 2e^{-\frac{nt^2}{2 \alpha^2}}$. For $t \geq \alpha^2 \beta$, the function $0.5 \alpha^2 \lambda^2 - \lambda t$ decreases monotonously as $\lambda$ increases since the gradient: $\lambda \alpha^2 - t$ is negative for $0 \leq \lambda < \beta, t \geq \alpha^2 \beta$ and thus the minimum is reached at $\lambda = \beta$. So, $0.5 \alpha^2 \lambda^2 - \lambda t < -\beta t + \frac{\beta^2 \alpha^2}{2} \leq -\beta t + \frac{\beta t}{2} = -\frac{\beta t}{2}$. The last inequality is due to $t \geq \alpha^2 \beta$. For $t > \alpha^2 \beta$, the strict inequality becomes an inequality, resulting in
	$\cfrac{E[e^{\lambda(X_i - \mu)/n}]}{e^{\lambda t/n}} \leq e^{-\frac{t\beta}{2}}$. Using the same argument for the confidence bound for the case that $t < \alpha^2 \beta$, we conclude the proof. \\
\end{proof}

A key point to notice in Theorem \ref{thm:subexpconcentration} is that sub-exponential variables exhibit gaussian-like tail decay for a small deviation $t$ in contrast to a slower fat tailed decay for larger $t$. 
Also, note that a given distribution may be sub-exponential with multiple settings of $\alpha$ and $\beta$.  To obtain the best concentration bounds, we would want to optimize these parameters, a task which will depend on the exact distribution being considered.

To give an example of how one can prove sub-exponentiality of random variables, we analyze the Laplace distribution. 

\begin{definition} \label{def:laplaceMGF}
	(Laplace MGF) If $X \sim $ Lap($\mu$, b) with probability density function $= \frac{1}{2b} e^{\frac{|x-\mu|}{b}}$, then $\E[e^{\lambda(X-\mu)}] = \frac{1}{1-b^2 \lambda^2}$ for $|\lambda| < \frac{1}{b}$ 
\end{definition}

\begin{proposition} \label{proposition:laplacesubexp}
	If $X \sim$ Lap($\mu$,b), then $X$ is a sub-exponential RV with $\alpha = b\sqrt{5.12}$ and $\beta = \frac{\sqrt{0.9}}{b}$. 
\end{proposition}
\begin{proof}
	Notice that $\frac{1}{1-x} \leq e^{2.56x}$ for $0 \leq x \leq 0.9$. The second inequality comes from basic calculations that conclude $e^{2.55x} \approx \frac{1}{1-x}$ for $x=0.9$, $e^{2.56x} >  \frac{1}{1-x}$, and the fact that $e^{2.56x}$ is always above the function $\frac{1}{1-x}$ for $0 \leq x \leq 0.9$. Based on the above inequality, for $X \sim$ Lap($\mu$,b), $\E[e^{\lambda(X-\mu)}] = \frac{1}{1-b^2 \lambda^2} \leq e^{2.56 b^2 \lambda^2} = e^{\frac{(\sqrt{5.12}b)^2 \lambda^2}{2}}$. Thus, $\E[e^{\lambda(X-\mu)}] \leq e^{\frac{\alpha^2 \lambda^2}{2}} $ for $\alpha = b\sqrt{5.12}$ and $|\lambda| < \beta = \frac{\sqrt{0.9}}{b}$. This concludes the proof.  
\end{proof}

Note that these constants for $\alpha$ and $\beta$ were not optimized in the above proof and the given values are only one parameter setting out of (infinitely) many that show that the Laplace distribution is sub-exponential.


\subsection{SAMPLE COMPLEXITY ANALYSIS OF UNCLIPPED MAVE}

In this section we assume that the loss is a sub-exponential random variable. Following the proof of Theorem \ref{thm_CMAVE}, it is not hard to notice that the only difference will be replacing the Hoeffding's confidence bound with the sub-exponential concentration bound. The following corollary states the result. 

\begin{corollary} \label{corollary_MAVE}
	Let $\{ s_1, \ldots, s_\nsampledstates \}$ be states sampled I.I.D according to $\weighting$ and that the number of rollouts for each state be $n$. 
	
	If the loss is a sub-exponential random variable with parameters $\alpha$ and $\beta$, with probability at least $1-\delta$ the following bound for unclipped MAVE holds:
	\begin{equation}
	\left| \CMAVEtrue(\vapprox,\vtrue)  -  \CMAVEapprox(\vapprox,\vsampled) \right| \le
	t  + \zeta.
	\end{equation}
	for $\zeta = 3 \rmax \left( \cfrac{1-\gamma^{\lenrollout}}{1-\gamma} \right) \cfrac{\log(6 \nsampledstates/ \delta)}{n} + \frac{\sum_{i=1}^{m}\std_i}{m} \sqrt{\cfrac{2\log(6 \nsampledstates/ \delta)}{n}} + \rmax \cfrac{\gamma^{\lenrollout}}{1-\gamma}$.
	
	Let $\sigma_1=\alpha \sqrt{\cfrac{2\log(4 \nexperiments/\delta)}{\nsampledstates}}$ and $\sigma_2=\cfrac{2\log(4 \nexperiments/\delta)}{\beta \nsampledstates}$.
	If $0<\sigma_1\leq \alpha^2 \beta$ and $0<\sigma_2\leq \alpha^2 \beta$, then $t=\sigma_1$. If $\sigma_1 > \alpha^2 \beta$ and $\sigma_2 > \alpha^2 \beta$, then $t=\sigma_2$.

\end{corollary}

\begin{proof}
	Let $X_i$ be the empirical loss for each state $i$ and the mean loss be $\mu$. Due to Theorem \ref{thm:subexpconcentration}, for $\nexperiments$ different empirical loss mean estimates and for $0 < \sigma_1 \leq \alpha^2 \beta$, setting ${Pr\left[ \left|\frac{1}{\nsampledstates}\sum^{\nsampledstates}_{i=1} X_i - \mu \right| \geq \sigma_1 \right] \leq 2e^{-\frac{\nsampledstates \sigma^2_1}{2 \alpha^2}} = \delta/2 \nexperiments}$, gives us $\sigma_1=\alpha \sqrt{\cfrac{2\log(4 \nexperiments/\delta)}{\nsampledstates}}$.
	
	For $\sigma_2 > \alpha^2 \beta$, setting $Pr \left[ \left|\frac{1}{\nsampledstates}\sum^{\nsampledstates}_{i=1} X_i- \mu \right| \geq \sigma_2 \right] \leq 2e^{-\frac{\nsampledstates \sigma_2 \beta}{2}} = \delta/2\nexperiments$, gives us $\sigma_2 = \cfrac{2\log(4 \nexperiments/\delta)}{\beta \nsampledstates}$. \\
	Based on the conditions for $\sigma_1$ and $\sigma_2$ to be valid, $t$ is chosen accordingly. Thus, using union bound over $\nexperiments$ empirical loss mean estimates, the total error probability is at most $\delta/2$. The rest of the results regarding $\zeta$ follows from Theorem $\ref{thm_CMAVE}$ since the bounding technique in its proof does not rely on clipping even though the loss is clipped. This later part gives an error probability of at most $\delta/2$ and so the total error probability is at most $\delta$.
\end{proof}

\subsection{SAMPLE COMPLEXITY ANALYSIS OF UNCLIPPED MSVE}

\begin{corollary} \label{corollary_MSVE}
	Let $\{ s_1, \ldots, s_\nsampledstates \}$ be states sampled I.I.D according to $\weighting$ and that the number of rollouts for each state be $n$. 
	
	If the loss is a sub-exponential random variable with parameters $\alpha$ and $\beta$, with probability at least $1-\delta$ the following bound for unclipped MSVE holds:
	\begin{equation}
	\left| \CMAVEtrue(\vapprox,\vtrue)  -  \CMAVEapprox(\vapprox,\vsampled) \right| \le
	t  + \zeta.
	\end{equation}
	for $\zeta = 3 \rmax^2 \left( \cfrac{1-\gamma^{\lenrollout}}{1-\gamma} \right)^2 \cfrac{\log(6 \nsampledstates/ \delta)}{n} + \frac{\sum_{i=1}^{m}\std_i}{m} \sqrt{\cfrac{2\log(6 \nsampledstates/ \delta)}{n}} + \rmax^2 \left( \cfrac{\gamma^{\lenrollout}}{1-\gamma} \right)^2$
	
	Let $\sigma_1=\alpha \sqrt{\cfrac{2\log(4 \nexperiments/\delta)}{\nsampledstates}}$ and $\sigma_2=\cfrac{2\log(4 \nexperiments/\delta)}{\beta \nsampledstates}$. If $0<\sigma_1\leq \alpha^2 \beta$ and $0<\sigma_2\leq \alpha^2 \beta$, then $t=\sigma_1$. If $\sigma_1 > \alpha^2 \beta$ and $\sigma_2 > \alpha^2 \beta$, then $t=\sigma_2$.   
	
\end{corollary}

\begin{proof}
	The same argument from Corollary \ref{corollary_MAVE} is applied here for choosing $t$ appropriately. Similarly, the rest of the results regarding $\zeta$ follows from Theorem $\ref{thm_CMSVE}$ since the bounding technique in its proof does not rely on clipping even though the loss is clipped. 
\end{proof}

	\section{ALGORITHM DETAILS}\label{app_algorithms}
	
	In this section, we provide additional details on the pseudocode in the main body, as well as providing the replacement for Algorithm \ref{alg_ebgstop} for other the losses discussed in the appendix. 
	
	\subsection{Sampling returns}
	
	To sample the returns to satisfy Assumption \ref{ass_one} for the discounted setting, we provide Algorithm \ref{alg_discounted}.
	We use the result in Lemma \ref{alg_discounted} to ensure Assumption \ref{ass_one} is satisfied. 
	
	\begin{algorithm}[t!]
		\caption{Sample truncated return to satisfy Assumption \ref{ass_one}}\label{alg_discounted}
		\begin{algorithmic}[1]
			\State $\triangleright$ Input $\epsilon,\delta, \relerror, \gamma$, state $s$
			\State $\triangleright$ Output a sampled return, $g$
			\State $\gammaprod = 1$
			\State $g \gets 0$
			\State $s_0 \gets s$
			\While{$\gammaprod > \errebstop (1-\gamma)/\rmax $}
			\State Sample next $s_{k+1}, r_{k+1}$, sampling the action according to $\pi( \cdot | s_k)$
			\State $g \gets g + \gammaprod r_{k+1}$
			\State $\gammaprod \gets \gammaprod \gamma$ 
			\EndWhile
			\Return $g$
		\end{algorithmic}
	\end{algorithm}
	
		\begin{algorithm}[t!]
		\caption{Empirical confidence interval using bootstrapping}
		\label{alg_opt}
		\begin{algorithmic}[1]
			\State $\triangleright$ Input number of sets to sample $k$ (e.g., k = 1000), and iteration $j$. 
			\State $\delta' \gets \frac{3}{(3/d_h)^{\alpha}} \cdot 100$  
			\State $D \gets$ randomly sample $k$ sets of size $j$ from the empirical distribution $\hat{F}$
			\State $\{g_1, \ldots, g_k\} \gets$ compute the means from the sets in $D$
			\State $c_{\delta'} \gets $ the $\delta'$'th percentile from $\{g_1, \ldots, g_k\}$
			\State $c_{100-\delta'} \gets $ the $(100 -\delta')$'th percentile from $\{g_1, \ldots, g_k\}$
			\State $c_j \gets \max(c_{\delta'}, c_{100-\delta'})$
			\State $\lb \gets \max(\lb, | \gbar | - c_j)$
			\State $\ub \gets \min(\ub, | \gbar | + c_j)$
		\end{algorithmic}
	\end{algorithm}
	
	There are a few other details that warrant explanation in the pseudocode for Algorithm \ref{alg_ebgstop}. The trajectory rollouts are of the appropriate lengths given by Lemma \ref{lem_discounted} to ensure the error due to truncation is sufficiently small. 
	For the empirical Bernstein inequality, we need to estimate the mean and variance of the sample truncated returns. We use a numerically stable approach to compute this sample mean and standard deviation, using Welford's algorithm \cite{welford1962note}.

\subsection{Sampling algorithm for CMAVE, CMSVE, MAVE and MSVE}

In this section, we present an incremental sampling algorithm (Algorithm \ref{alg_appLoss}) that can be used to sample states with their values and hence guarantee that the high probability errors of sub-exponential MAVE, MSVE and clipped MAVE, MSVE are bounded by a desired preset amount $\epsilon$. This algorithm would be called in Algorithm 1, in place of Algorithm \ref{alg_ebgstop}. Since a given error can be satisfied with different combinations of $\nsampledstates$ ---the number of sampled states---and $n$---the number of rollouts per state---one option for MSVE and MAVE is to pick $\nsampledstates$ such that the error contributed by the sub-exponential bound is atmost $\alpha^2 \beta$ to take advantage of the subgaussian tail decay. Such a choice corresponds to $\nsampledstates = \lceil \frac{2 \log(4 \nexperiments / \delta)}{\alpha^2 \beta^2} \rceil$. For CMAVE, CMSVE, we suggest fixing $m$ beforehand depending on $\cliperror$, $\epsilon$ and how costly it is to sample more rollouts compared to sampling states. We leave other selection criteria for future work.

\begin{algorithm}[t!]
	\caption{High confidence $\vsampled$ estimator for clipped losses}
	\label{alg_appLoss}
	\begin{algorithmic}[1]
		\State $\triangleright$ Input $\epsilon,\delta, \nsampledstates, \nexperiments, \alpha, \beta$
		\State $\triangleright$ Compute the values $\vsampled$ once offline and store for repeated use.
		\State $\triangleright$ If using CMAVE/MAVE, set $\vmax = \rmax \left( \cfrac{1-\gamma^{\lenrollout}}{1-\gamma} \right)$
		\State $\triangleright$ If using CMSVE/MSVE, set $\vmax = \rmax^2 \left( \cfrac{1-\gamma^{\lenrollout}}{1-\gamma} \right)^2$
		\State $\triangleright$ If using MAVE/MSVE, set $\nsampledstates = \lceil \frac{2 \log(4 \nexperiments / \delta)}{\alpha^2 \beta^2} \rceil$ and $\zeta \gets \epsilon - \alpha \sqrt{\cfrac{2\log(4 \nexperiments/\delta)}{\nsampledstates}}$
		\State $\triangleright$ Else for CMAVE/CMSVE: $\zeta \gets \epsilon - \sqrt{\frac{\log(4 \nexperiments/\delta) \cliperror^2}{2\nsampledstates}}$
		\State $\triangleright$ For states $i=1,..,m$ initialize:
		\State $\gbar_i \gets 0$, $M_i \gets 0$ 
		\State $j_i \gets 1$, $h_i \gets 0$, $\alpha_i \gets 1$, $x_i \gets 1$
		\State $\beta \gets 1.1$, $p \gets 1.1$, $\zeta \gets \epsilon - \sqrt{\frac{\log(4 \nexperiments/\delta) \cliperror^2}{2\nsampledstates}}$
		\While {True}
		\For {$i \in \{1,..,m\}$}
		\State $g_i \gets $ Sampled return from state $i$ of length $\lenrollout$
		\State $\Delta_i \gets g_i - \gbar_i$
		\State $\gbar_i \gets \gbar_i + \frac{\Delta_i}{j_i}$
		\State $M_i \gets M_i + \Delta_i (g_i - \gbar_i)$
		\State $\sigma_i \gets \sqrt{M_i/j_i}$
		\State $\triangleright$ Compute the confidence interval
		\If{$j_i \ge \floor{\beta^{h_i}}$}
		\State $h_i \gets h_i + 1$
		\State $\alpha_i \gets \floor{\beta^{h_i}}/\floor{\beta^{h_i-1}}$
		\State $x_i \gets -\alpha_i \log \cfrac{\delta(p-1)}{6 \nsampledstates p h_i^p} $  
		\EndIf
		\State $c_i \gets \sigma_i \sqrt{\frac{2x_i}{j_i}} + \frac{3 \vmax x_i}{ j_i}$
		\State $j_i = j_i+1$
		\EndFor
		\If{$\frac{\sum_{i=1}^{m} c_i}{m} \leq \zeta$}
		\State $\triangleright$ For all states $i=1,..,m$ :
		\State $\vsampled(i) \gets \gbar_i$
		\State \textbf{return} $\vsampled$
		\EndIf
		\EndWhile
	\end{algorithmic}
\end{algorithm}

\subsection{Computation of optimal intervals}
	For completeness, we include how we used bootstrapping to compute the intervals to provide a similar
	stopping rule to EBGStop. The algorithm is the same, except in how the confidence intervals are computed. 
	We first generate a large batch of data, to act as the empirical distribution. We could simply sample sets of size $j$ repeatedly, from the simulator, to get a sense of variability of sample averages. However, we choose to sample a very large batch of data upfront, to reduce the computational burden of the procedure. On each step, a large number $k$ of set of $j$ return samples are drawn, and their sample average computed to obtain the spread of values.
	Then the percentile corresponding to $\delta$ is computed, to provide a high-confidence estimate of a lower and an upper bound on the true values. 
	This approach to computing the true confidence interval is given in Algorithm \ref{alg_opt}.
	We sampled an batch of $10^7$ returns for each state, to provide the empirical distribution, and set $k = 1000$.
	
	This approach is not a suitable strategy to get high confidence estimates, because it requires a very large number of samples. Rather, we only used this strategy as a comparison, to provide a close approximation to the true confidence intervals, and so obtain best-case sampling numbers. This allowed us to evaluated the impact of the looseness of our bounds, in terms of how many extra samples are generated. 
	
\end{document}